\documentclass[letterpaper,11pt]{article}
\usepackage[margin=1in]{geometry}
\usepackage[bookmarks, colorlinks=true, plainpages = false, colorlinks=true,
            citecolor=blue,
            linkcolor=blue,
            anchorcolor=red,
            urlcolor=blue]{hyperref}
\usepackage{url}
\usepackage{amsmath, amsfonts,amsthm,amssymb,bm, verbatim,dsfont,mathtools}
\usepackage{color,graphicx,appendix}
\usepackage{etoolbox}
\usepackage{array}
\usepackage{multirow}
\usepackage{float}

\usepackage{xr,xspace}
\usepackage[dvipsnames]{xcolor}
\usepackage{todonotes}
\usepackage{paralist}
\usepackage{enumitem}
\usepackage{booktabs}

\usepackage{caption,subcaption,soul}
\usepackage{algorithm}
\usepackage{algpseudocode}
\makeatletter
\usepackage{romanbar}
\usepackage{mathpazo}
\usepackage[sort]{natbib}

\newtheorem{theorem}{Theorem}
\newtheorem{corollary}{Corollary}
\newtheorem{lemma}{Lemma}

\newtheorem{proposition}{Proposition}

\newtheorem{problem}{Problem}

\newtheorem{definition}{Definition}

\newtheorem{fact}{Fact}
\newtheorem*{claim}{Claim}
\newtheorem*{customthm*}{\thistheoremname}
\newcommand{\thistheoremname}{} %

\newcommand \E[1]{\mathbb{E}[#1]}

\newcommand{\reals}{{\mathbb{R}}}

\newcommand{\diff}{{\mathrm d}}
\newcommand{\Prob}{\mathbb{P}}
\def\Var{\mathrm{Var}}
\def\E{\mathbb{E}}
\def\Cov{\mathrm{Cov}}

\newcommand{\calA}{{\mathcal{A}}}

\newcommand{\calL}{{\mathcal{L}}}

\newcommand{\calN}{{\mathcal{N}}}

\newcommand{\eps}{\varepsilon}
\newcommand{\Ber}{\mathrm{Ber}}

\newif\ifhideproofs
\ifhideproofs
\usepackage{environ}
\NewEnviron{hide}{}

\fi

\newif\ifdraft
\drafttrue

\begin{document}
\title{Can R\'enyi Accountants Beat Lossless Accountants for DP-SGD?}
\author{Jinho Bok\thanks{Department of Statistics and Data Science, University of Pennsylvania. Email: \href{mailto:jinhobok@wharton.upenn.edu}{\color{black}{\texttt{jinhobok@wharton.upenn.edu}}}}}
\date{October 12, 2026}
\maketitle
\thispagestyle{empty}

\begin{abstract}
Recent frameworks in differential privacy (DP), such as privacy loss distributions and $f$-DP, provide lossless representations of the privacy leakage of mechanisms. These approaches are often considered to yield tighter $(\eps, \delta)$-DP guarantees than those obtained from relaxed notions such as R\'enyi DP. But is that always the case? We study this question for differentially private stochastic gradient descent (DP-SGD) with fixed-size batches under replace adjacency. Our key observation is that in this setting, privacy accounting based on R\'enyi DP can use a pair of worst-case distributions (i.e., a dominating pair) that differs from the one used in the lossless frameworks. In particular, we characterize the worst-case distributions as two-dimensional Gaussian mixtures and derive an improved R\'enyi DP bound for DP-SGD. Numerical and asymptotic comparisons show that across various hyperparameter settings, a R\'enyi DP accountant based on the new bound can yield better privacy guarantees than those obtained from the lossless frameworks.
\end{abstract}

\newpage
\tableofcontents
\thispagestyle{empty}
\clearpage

\setcounter{page}{1}

\section{Introduction}
Differential privacy (DP) \citep{dmns06} is widely regarded as a gold standard for data privacy and has been adopted in numerous areas \citep{census, wiki, gboard, privacysurvey, privacyai}. A prominent problem in DP across different application domains is \emph{private optimization}, where the goal is to train models while preserving the privacy of each individual's information.

Differentially private stochastic gradient descent (DP-SGD) \citep{dpsgd-dl} is the de facto standard algorithm for private optimization. Hence, quantifying the privacy leakage after multiple iterations (i.e., \emph{privacy accounting}) of DP-SGD is a fundamental problem. Formally, this can be done by bounding the divergences between the outputs from adjacent datasets, i.e., two datasets 
such that one is obtained by replacing an individual record from the other. In the literature, there are two prototypical choices for the divergences: R\'enyi divergences and hockey-stick divergences. As frameworks for privacy accounting, the former is commonly called R\'enyi DP \citep{mir17}, whereas the latter is often identified with equivalent frameworks such as privacy loss distributions (PLD) \citep{dwork2016concentrated, meiser2018tight}, privacy loss random variables (PRV) \citep{gopi2021numerical}, and $f$-DP \citep{drs22}.

A conventional view is that the privacy guarantees\footnote{Throughout, we compare privacy guarantees from different approaches in $(\eps, \delta)$-DP (see Definition~\ref{def:dp}).} based on hockey-stick divergences are tighter than those from R\'enyi DP \citep{koskela2020fft, gopi2021numerical, drs22}. This is because hockey-stick divergences (and equivalent notions for privacy accounting) capture the entire privacy characteristic of a mechanism that can be losslessly converted to other DP notions, whereas R\'enyi DP is inherently lossy as a relaxed notion of DP. However, an underlying assumption of this argument is that they use the same pair of worst-case distributions (i.e., a \emph{dominating pair}) to characterize the privacy leakage of a mechanism. We question this premise and ask: can R\'enyi DP provide better privacy guarantees by using a different pair of distributions?

\subsection{Our contributions}
We answer this question in the affirmative for DP-SGD with fixed-size batches (i.e., sampling without replacement). In particular, we propose an improved R\'enyi DP bound in this setting, using certain two-dimensional Gaussian mixtures as worst-case distributions (Section~\ref{sec:renyiaccountant}). This R\'enyi DP bound is characterized as the value of a one-dimensional maximization problem that can in general be solved numerically, and under certain conditions admits simple, explicit solutions.
Through numerical comparisons (Section~\ref{sec:numerics}), we show that across various hyperparameter settings, our new R\'enyi DP bound yields tighter privacy guarantees than those from hockey-stick divergences. We note, however, that such improvement is not uniform over all 
choices of DP-SGD hyperparameters, and the privacy
bounds in their original analytical forms are rather difficult to compare. We provide asymptotic comparisons (Section~\ref{sec:asymp}) where we compare limits of privacy parameters in the regime where they are convergent, illustrating the potential advantages of our new privacy bound based on analytical results. Such comparisons are also done against existing R\'enyi DP bounds \citep{wang2021rdp, birrell2024rdp} where we demonstrate uniform improvements over theirs.

Conceptually, our results reveal an interesting trade-off between the granularity of information used and the quality of worst-case distributions for privacy accounting. Indeed, hockey-stick divergences and their equivalent accounting formalisms---which, for example, characterize the privacy curve \citep{meiser2018tight} or the Type I/II error curve \citep{drs22}---are based on the entire privacy characteristic and thus operate at the level of ``function''. The vast amount of information needed in those frameworks handles key components of privacy accounting---e.g., composition of mechanisms, conversion to other DP notions---in a lossless manner, but also limits the corresponding worst-case distributions to simultaneously satisfy many conditions (i.e., inequalities holding at all $\gamma$, as in Definition~\ref{def:dominatingpairs}). On the other hand, R\'enyi DP only characterizes a certain moment bound and works at the level of ``functional''; while such information is coarse and may incur loss---e.g., for conversion to other DP notions, the corresponding worst-case distributions only need to satisfy a coarse condition (i.e., an inequality at a single $\alpha$; see also Definition~\ref{def:dominatingpairs}). We believe that this discrepancy between the two frameworks merits further investigation, and hope that our case study in DP-SGD here motivates such inquiry.

\subsection{Related work}
\paragraph*{Privacy accounting.} 
Bounding the privacy leakage of mechanisms over composition has been a central topic in DP.
The frameworks based on R\'enyi divergences are collectively called R\'enyi DP \citep{mir17}, and related notions such as concentrated DP \citep{dwork2016concentrated} and its variants \citep{bs16, bun2018tcdp} have also been widely studied. While R\'enyi DP is desirable due to its analytical convenience (e.g., composition expressed as a sum; see Fact~\ref{fact:composition}), it is also known to be inherently lossy when converted to other notions of DP including $(\eps, \delta)$-DP (see, for example, \citep[Proposition B.7]{drs22}). 

On the other hand, the second line of frameworks---e.g., PLD \citep{dwork2016concentrated, meiser2018tight, sommer2019privacy}, PRV \citep{gopi2021numerical}, $f$-DP \citep{drs22}---are lossless when converted to any other notion of DP.
At a fundamental level, they are all based on a notion called Blackwell order \citep{blackwell1951comparison} and hence are equivalent to each other, sharing the same dominating pairs and privacy guarantees. See, for example, \citep[Section 2]{regehr2026privacy} for a summary of such facts. Despite their lossless property, the resulting privacy bounds are often analytically cumbersome, and their evaluations often rely on different numerical approaches \citep{koskela2020fft, doroshenko2022connect, ghazi2022faster, zdw22, alghamdi2023saddlepoint}.

\paragraph*{Privacy of DP-SGD.} 
For privacy accounting of DP-SGD, we follow the mainstream approach: first obtain a privacy bound for a single iteration of DP-SGD, and then extend that for multiple iterations by (adaptive) composition. Among this line of work, the original paper on DP-SGD \citep{dpsgd-dl} proposed a framework known as the moments accountant, which is now subsumed by R\'enyi DP \citep{mir17}. Since then, several works have provided tighter privacy guarantees for DP-SGD under R\'enyi DP \citep{mironov2019renyi, wang2021rdp, birrell2024rdp} as well as under the lossless frameworks \citep{bdls20, drs22}. It is also worth mentioning that recent literature has developed alternative approaches \citep{fmtt18, ys22, altschuler2024noisy, bok2024shifted, kong2024privacy, chien2025privacy}, %
assuming that only the final iterate of DP-SGD is released. While those results provide improved privacy bounds (sometimes even \emph{convergent} as the number of iterations increases), they are also based on restrictive assumptions such as loss functions being convex.
In that sense, they provide a complementary perspective for the former approach.

\paragraph*{Subsampling schemes for DP-SGD.} There are different ways to subsample the random batch in DP-SGD; a popular example is the Poisson subsampling \citep{dpsgd-dl}, where each index is included in the batch independently. Furthermore, developing and analyzing new subsampling schemes for DP-SGD has gained attention in recent years \citep{chua2025balls, shenfeld2025random, dong2026less}. We deliberately choose the combination of fixed-size batches and replace adjacency (Definition~\ref{def:dp}), because it is a practically relevant and analytically clean setting that exhibits a discrepancy between different privacy accountants in terms of the worst-case distributions. 
\citep{lebeda2025pitfalls} discuss several nuances around subsampling schemes and adjacency notions for DP-SGD, highlighting how different (but seemingly similar) choices have disparate effects for privacy accounting. In a similar vein, we study the implications of a particular set of DP-SGD specifications for privacy accounting.

\subsection{Notations}
We denote $[n] \triangleq \{1, \dots, n\}$ and $a \vee b \triangleq \max\{a, b\}, a \wedge b \triangleq \min\{a, b\}$. For a distribution $P$, $P^{\otimes T}$ denotes the product distribution of $T$ copies of $P$. $\phi$ and $\Phi$ respectively denote the PDF and CDF of $\calN(0, 1)$; $\Ber(p)$ denotes the Bernoulli distribution with success probability $p$. We often write the distribution itself to denote an independent random variable/vector that is distributed as such; for example, $a \Ber(p) + \calN(0, I_d)$ denotes a mixture distribution of $\calN(a, I_d)$ and $\calN(0, I_d)$ with respective mixture weights $p$ and $1-p$. We use standard big-$O$ notation, with subscripted variables under $O$ denoting factors that may depend on those variables.

\section{Preliminaries}\label{sec:prelim}

\subsection{Privacy accounting via dominating pairs}\label{subsec:privacyaccounting}
As a formal notion, DP evaluates the effect of an individual's data on the output of a mechanism by statistically comparing the outputs from adjacent datasets.

\begin{definition}[Replace adjacency and $(\eps, \delta)$-DP]\label{def:dp} Let $S \triangleq (z_1, \dots, z_n)$ and $S' \triangleq (z'_1, \dots, z'_n)$. Then $S$ and $S'$ are \emph{adjacent} if there exists $i^* \in [n]$ such that $z_i = z'_i$ for all $i \in [n] \setminus \{i^*\}$. A mechanism $\calA$ is \emph{$(\eps, \delta)$-DP} if for all adjacent datasets $S, S'$ and events $E$,
\[\Prob(\calA(S) \in E) \leq \exp(\eps)\Prob(\calA(S') \in E) + \delta\,.
\]   
\end{definition}
Intuitively, smaller values of $\eps$ or $\delta$ mean that the two distributions $\calA(S)$ and $\calA(S')$ are closer, and hence harder to distinguish. Formally, such privacy guarantees can be obtained by bounding the divergences between the two outputs. %

\begin{definition}[R\'enyi divergence and hockey-stick divergence] Let $P, Q$ be distributions. For $\alpha > 1$, the  \emph{R\'enyi divergence of order $\alpha$} between $P$ and $Q$ is defined as
\[D_\alpha(P \| Q) \triangleq \frac{1}{\alpha - 1} \log \int \left(\frac{\diff P}{\diff Q}(\omega)\right)^{\alpha} \diff Q(\omega)\,.
\]
For $\gamma \geq 0$, the \emph{hockey-stick divergence of order $\gamma$} between $P$ and $Q$ is defined as 
(with $x_+ \triangleq x \vee 0$)
\[H_\gamma(P \|Q) \triangleq \int \left(\frac{\diff P}{\diff Q}(\omega) - \gamma\right)_+ \diff Q (\omega)\,.
\]
\end{definition}
For any type of divergence under consideration, the corresponding $(\eps, \delta)$-DP guarantee can be obtained by first identifying a dominating pair \citep{sommer2019privacy, zdw22}, and then converting the corresponding divergence bound to a $(\eps, \delta)$-DP bound. Dominating pair can be considered as ``worst-case'' distributions for quantifying the privacy leakage of a mechanism.

\begin{definition}[Dominating pairs]\label{def:dominatingpairs} Let $\calA$ be a mechanism. Then $(P, Q)$ is an \emph{(order-$\alpha$) R\'enyi dominating pair} for $\calA$ if for all adjacent datasets $S$ and $S'$,
\[D_\alpha(\calA(S) \| \calA(S')) \leq D_\alpha (P \| Q)\,.
\]
$(P, Q)$ is a \emph{hockey-stick dominating pair}\footnote{In the literature, dominating pairs usually refer to the hockey-stick dominating pairs. Also, in principle, privacy accounting based on R\'enyi divergences can be formulated without the notion of dominating pairs. Here we choose to name and define dominating pairs according to the divergences, in order to highlight the conceptual difference between the corresponding privacy 
accountants.} for $\calA$ if for all adjacent datasets $S$ and $S'$,
\[H_\gamma(\calA(S) \| \calA(S')) \leq H_\gamma (P \| Q) \text{ for all }\gamma \geq 0.
\]
\end{definition}
A central concept in DP is composition, where a mechanism is built by composing other mechanisms. The notion of dominating pairs is particularly useful for handling composition, because the cumulative privacy leakage can be accounted by taking the corresponding product distributions as a dominating pair.

\begin{fact}[Composition using dominating pairs; {\citep[Proposition 1]{mir17}}, {\citep[Theorem 10]{zdw22}}]\label{fact:composition} For each $k \in [T]$, let $(P_k, Q_k)$ be an order-$\alpha$ R\'enyi dominating pair for mechanism\footnote{Here, each mechanism $\calA_k$ admits an additional input $\theta$ that can be considered as an initialization or an output from a previous iteration.} $\calA_k(\cdot, \theta)$ for all $\theta$. 
Then $(\otimes_{k=1}^T P_k, \otimes_{k=1}^T Q_k)$ is an order-$\alpha$ R\'enyi dominating pair for the composed mechanism (for any initialization $\theta_0$) $\calA \triangleq \calA_T \circ \calA_{T-1} \circ \dots \circ \calA_1:  S \mapsto  \calA_T(S, \calA_{T-1}(S, \dots , \calA_1(S, \theta_0))\dots)$,
i.e.,
\begin{align*}
    D_\alpha(\calA(S) \| \calA(S')) \leq D_\alpha(\otimes_{k = 1}^T P_k \| \otimes_{k=1}^T Q_k) = \sum_{k=1}^T D_\alpha(P_k \| Q_k)\,.
\end{align*}
Similarly, if $(P_k, Q_k)$ is a hockey-stick dominating pair for $\calA_k(\cdot, \theta)$ for all $\theta$, 
then $(\otimes_{k=1}^T P_k, \otimes_{k=1}^T Q_k)$ is a hockey-stick dominating pair for $\calA$, i.e.,
\[H_\gamma(\calA(S) \| \calA(S')) \leq H_\gamma(\otimes_{k=1}^T P_k \| \otimes_{k=1}^T Q_k) \text{ for all } \gamma \geq 0\,.
\]
\end{fact}
The divergence bounds for the dominating pairs can be converted to $(\eps, \delta)$-DP bounds based on certain formulae \citep{mir17, balle2020hypothesis, asoodeh2021conversions, zdw22, riess2026optimal}.

\subsection{R\'enyi accountants versus lossless accountants}
Throughout, we refer to the procedures of privacy accounting explained in Section~\ref{subsec:privacyaccounting} as \emph{R\'enyi accountants} if they are based on R\'enyi divergences; similarly, we refer to the corresponding procedures based on hockey-stick divergences as \emph{lossless accountants}. The latter term collectively denotes several accountants, including those based on PLD, PRV, $f$-DP, characteristic functions \citep{zdw22}, and more; such abstraction is justified by the fact that the underlying notions of dominating pair and composition are all mathematically equivalent \citep{zdw22}.
Both R\'enyi and lossless accountants are widely used and have been deployed in DP libraries such as Opacus \citep{opacus} and JAX-Privacy \citep{jaxprivacy}. 

A distinctive feature of lossless accountants is that they allow lossless conversion to $(\eps, \delta)$-DP, as demonstrated by the following necessary and sufficient condition.

\begin{fact}[{\citep[Lemma 5]{zdw22}}] $\calA$ is $(\eps, \delta)$-DP if and only if $H_{\exp(\eps)}(\calA(S) \| \calA(S')) \leq \delta$ for all adjacent datasets $S$ and $S'$.
\end{fact}

On the other hand, such ``if and only if'' statement does not hold for R\'enyi divergences, and the conversions to $(\eps, \delta)$-DP bounds are inherently lossy; see, for example, \citep{bs16, riess2026optimal}. This implies that if a R\'enyi accountant and a lossless accountant both use the same pair of distributions as their dominating pair, the latter would provide a privacy guarantee that is at least as good as (and often strictly better than) what the former would provide. Furthermore, for several mechanisms such a gap cannot be improved because both accountants share the same tight (i.e., the best possible) dominating pair in those cases \citep{mir17, zdw22}.

\subsection{Differentially private stochastic gradient descent}
Throughout, we consider a standard version of DP-SGD where at each iteration, a fixed-size batch is chosen uniformly at random (i.e., sampling without replacement). 

\begin{definition}[DP-SGD with fixed-size batches]\label{def:dp-sgd} Given a dataset $(z_1, \dots, z_n)$, loss functions $\theta \in \reals^d \mapsto \ell(\theta, z_i)$ for $i \in [n]$, and initialization $\theta_0$, DP-SGD is a mechanism such that for $k = 0, 1,\dots, T-1$,
\begin{enumerate}[leftmargin=*]
    \item samples $b$ distinct elements $B_k$ from $[n]$ uniformly at random;
    \item updates as
    \[\theta_{k+1} \gets \theta_k - \eta_k\left(\frac{1}{b}\sum_{i \in B_k} \mathrm{clip}_C(\nabla \ell(\theta_k, z_i)) + \frac{\sigma C}{b}\xi_k\right)\,,
    \]
    where $\mathrm{clip}_C(x) \triangleq x / (1 \vee \|x\|/C)$, and $\xi_k \sim \calN(0, I_d)$ are independent;
\end{enumerate}
and outputs $\theta_T$.
\end{definition}

\section{Dominating pairs for DP-SGD}\label{sec:dominatingpairs}

What do the dominating pairs for one iteration of DP-SGD look like? Intuitively, the key point is whether the special index $i^*$ (see Definition~\ref{def:dp}) and the associated point $z_{i^*}$ is used.
It is straightforward to see that $i^*$ is included in the batch with probability $p \triangleq b/n$. Also, if $z_{i^*}$ is replaced with some other point $z$ ($z'_{i^*}$ from adjacent dataset, or $z_i$ with $i \neq i^*$ from batch subsampling), the gradient can differ by at most $\|\mathrm{clip}_C(\nabla \ell(\theta_t, z_{i^*})) - \mathrm{clip}_C(\nabla \ell(\theta_t, z))\| / b \leq 2C/b$. By scaling so that the Gaussian noise has unit variance, this difference in gradient becomes at most $\Delta \triangleq 2/\sigma$. Combining these arguments, one can consider the following family of candidates for dominating pairs:
\begin{equation}\label{eq:gaussianmixturefamily}
\begin{aligned}
    \{(a\Ber(p) + \calN(0, I_d), b\Ber(p) + \calN(0, I_d)):  \|a\| \leq \Delta, \|b\| \leq \Delta, \|a - b\|\leq \Delta \}\,.
\end{aligned}
\end{equation}
For R\'enyi accountants, \eqref{eq:gaussianmixturefamily} is indeed sufficient for characterizing dominating pairs in the following sense:
\begin{lemma}[{\citep[Equation 56]{birrell2024rdp}}] Define
\begin{equation}\label{eq:isoproblem}
\begin{aligned}
    \rho(\alpha, p, \Delta) \triangleq \sup_{\substack{\|a\| \leq \Delta, \\ \|b\| \leq \Delta, \\ \|a - b\| \leq \Delta}} D_\alpha(a\Ber(p) + \calN(0, I_d) \|     b\Ber(p) + \calN(0, I_d))\,. 
\end{aligned}
\end{equation}
Then $(P, Q)$ with $D_\alpha(P \| Q) = \rho(\alpha, p, \Delta)$ is an order-$\alpha$ R\'enyi dominating pair for DP-SGD (Definition~\ref{def:dp-sgd}), where $\Delta = 2/\sigma$.
\end{lemma}
It can be shown that for each $(\alpha, p, \Delta)$, the right hand side of \eqref{eq:isoproblem} (as the maximum of a continuous function on a compact set) attains its maximum at some $(a, b) = (a^*, b^*)$. This implies that one can take a single pair $(a^*\Ber(p) + \calN(0, I_d), b^*\Ber(p) + \calN(0, I_d))$ as a R\'enyi dominating pair.
On the other hand, for lossless accountants, at least two distinct elements from \eqref{eq:gaussianmixturefamily} need to be considered (see, e.g., \citep[Section C.4]{zdw22} or \citep[Section 7]{lebeda2025pitfalls}).
\begin{lemma}\label{lem:hockeystick-necessary} Let $(P, Q)$ be a hockey-stick dominating pair for DP-SGD (Definition~\ref{def:dp-sgd}). Then for $\Delta = 2/\sigma$,
\begin{equation}\label{eq:hockeystick-dominatingpair-g1}
    H_\gamma(P \| Q) \geq H_\gamma(\Delta \Ber(p) + \calN(0, 1) \| \calN(0, 1))
\end{equation}
for all $\gamma \geq 1$, and
\begin{equation}\label{eq:hockeystick-dominatingpair-s1}
    H_\gamma(P \| Q) \geq H_\gamma(\calN(0, 1) \| \Delta \Ber(p) + \calN(0, 1))
\end{equation}
for all $0 < \gamma < 1$.
\end{lemma}
In particular, it follows that neither $(\Delta \Ber(p) + \calN(0, 1), \calN(0, 1))$ nor $(\calN(0, 1), \Delta \Ber(p) +\calN(0, 1))$ is by itself a hockey-stick dominating pair; recall from Definition~\ref{def:dominatingpairs} that a dominating pair should simultaneously satisfy the inequalities for all $\gamma \geq 0$. Nonetheless, the conditions \eqref{eq:hockeystick-dominatingpair-g1} and \eqref{eq:hockeystick-dominatingpair-s1} are sufficient, and accordingly lossless accountants can provide the following privacy guarantee for DP-SGD \citep{drs22, zdw22, lebeda2025pitfalls}.
\begin{proposition}[Privacy guarantee from lossless accountants; {\citep[Corollary 32]{zdw22}}]\label{prop:privacy-lossless} For each $p \in (0, 1)$ and $\sigma > 0$, there exists a hockey-stick dominating pair $(\widetilde{P}, \widetilde{Q})$ for DP-SGD (Definition~\ref{def:dp-sgd}) such that \eqref{eq:hockeystick-dominatingpair-g1} holds with equality for all $\gamma \geq 1$ and \eqref{eq:hockeystick-dominatingpair-s1} holds with equality for all $0 < \gamma < 1$. 
Consequently, if $H_{\exp(\eps)}\left(\widetilde{P}^{\otimes T} \| \widetilde{Q}^{\otimes T}\right) \leq \delta$
then DP-SGD with $T$ iterations is $(\eps, \delta)$-DP.
\end{proposition}

The discussion so far highlights the key difference between R\'enyi and lossless accountants that we study in this work: R\'enyi accountants can take a single element from \eqref{eq:gaussianmixturefamily} as a dominating pair, whereas lossless accountants need\footnote{This ``necessity'' argument assumes using the mainstream approach of using dominating pair as explained in Section~\ref{subsec:privacyaccounting}. Whether such approach can be bypassed is an open question; see also Sections~\ref{subsec:privacyaccounting} and \ref{sec:discussion} for discussion.} to take two distinct elements from \eqref{eq:gaussianmixturefamily} and ``combine'' them as in Proposition~\ref{prop:privacy-lossless}. In particular, this suggests that different dominating pairs can be used for those accountants.

\section{Improved R\'enyi accountant for DP-SGD}\label{sec:renyiaccountant}

\subsection{Characterization of R\'enyi DP bound}
We show that for the problem \eqref{eq:isoproblem} of maximization over Gaussian mixtures, it suffices to consider two-dimensional Gaussians whose means $a, b$ and $0$ form an isosceles ($\|a-0\| = \|a - b\| = \Delta$); this effectively reduces the problem \eqref{eq:isoproblem} over $d$-dimensional vectors to a one-dimensional problem. In words, this implies that for each iteration of DP-SGD, there exist two-dimensional Gaussian mixtures that serve as a dominating pair. See Figure~\ref{fig:renyi-mixturegeometry} for an illustration.

\begin{figure}[t]
    \centering
    \includegraphics[width=0.6\textwidth]{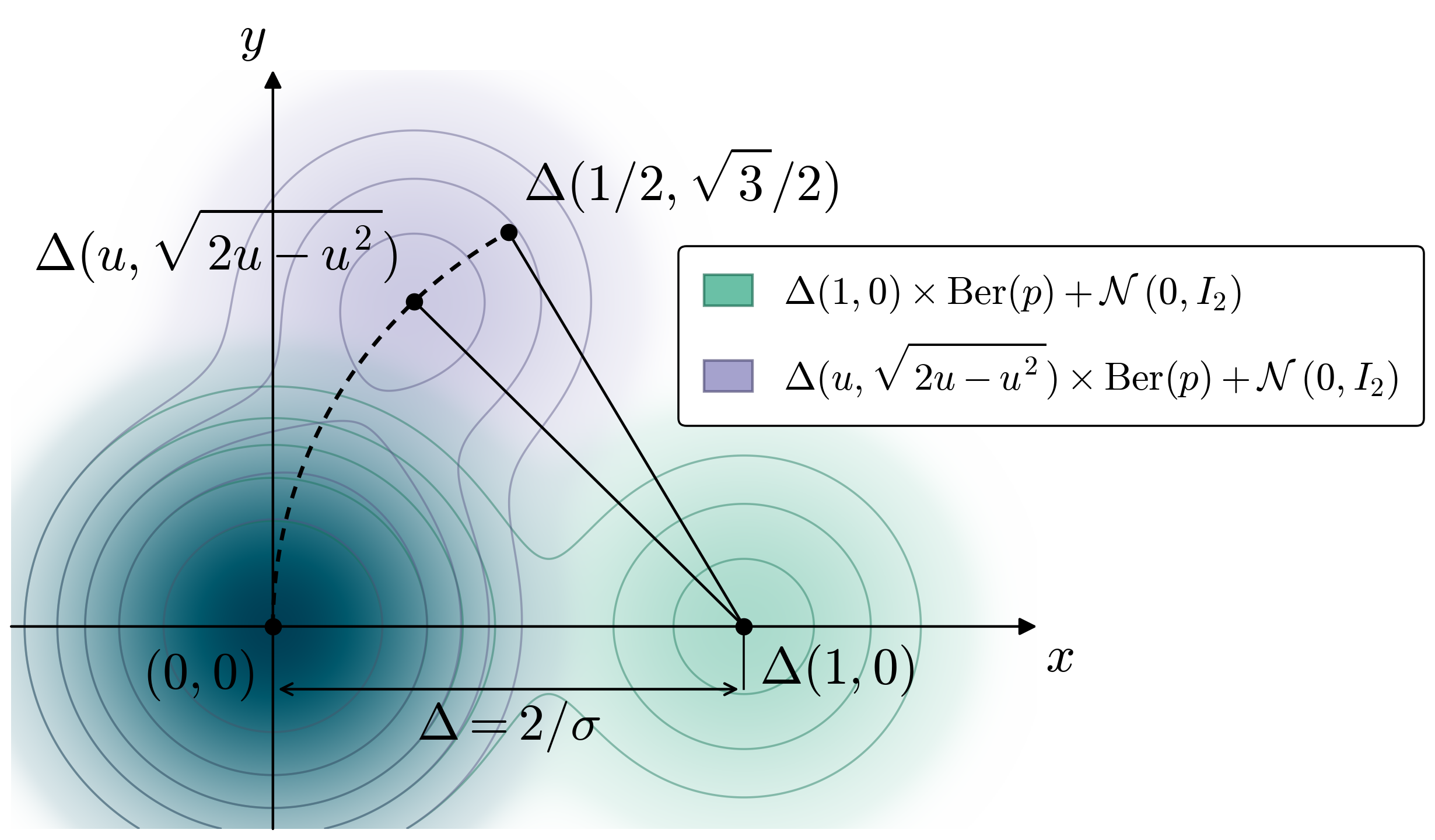}
    \caption{Illustration of the Gaussian mixtures that serve as a R\'enyi dominating pair for DP-SGD (see \eqref{eq:1dopt}); here $p = 0.2, \Delta = 4.5, u = 0.3$. The points $(0, 0)$, $\Delta(1, 0)$, and $\Delta(u, \sqrt{2u - u^2})$ (each corresponding to the mean of the mixture component) form an isosceles.} %
    \label{fig:renyi-mixturegeometry}
\end{figure}

\begin{theorem}\label{thm:1dopt} If $d \geq 2$, for \eqref{eq:isoproblem} we have
\begin{equation}\label{eq:1dopt}
    \rho(\alpha, p, \Delta) = \max_{u \in [0, 1/2]} D_\alpha\Big(\Delta(1, 0)\times \mathrm{Ber}(p) + \calN(0, I_2)  \Big \|\Delta(u, \sqrt{2u - u^2}) \times \mathrm{Ber}(p) + \calN(0, I_2)\Big)\,.
\end{equation}
\end{theorem}

\begin{proof}[Proof sketch] The proof proceeds by reducing \eqref{eq:isoproblem} to simpler problems. It is straightforward to check that the R\'enyi divergence in \eqref{eq:isoproblem} for each $a, b$ only depends on the covariance matrix $\begin{pmatrix}
    \|a\|^2 & \langle a, b \rangle \\ \langle a, b \rangle & \|b\|^2
\end{pmatrix}$. Thus, we can assume $d = 2$. Furthermore, we show that it suffices to consider such covariance matrices that are maximal in the Loewner order. This leaves only a few possible cases, from which we derive \eqref{eq:1dopt} by ruling out most of the cases based on direct comparisons. See Appendix~\ref{subsec:1dopt} for a complete proof.
\end{proof}

Theorem~\ref{thm:1dopt} provides a one-step R\'enyi DP bound. Our proposed R\'enyi accountant is based on this one-step bound, and uses R\'enyi composition (Fact~\ref{fact:composition}) as well as conversion from R\'enyi DP to $(\eps, \delta)$-DP from \citep[Theorem 20]{balle2020hypothesis} to calculate the following privacy bound for DP-SGD.\footnote{Recently, \citep{riess2026optimal} provided an optimal procedure for converting R\'enyi DP guarantees to $(\eps, \delta)$-DP guarantees, which would in principle provide a tighter bound. Nonetheless, we focus on the formula from \citep{balle2020hypothesis}, as it has a simple closed form, is convenient for implementation (see Proposition~\ref{prop:unimodal}), and can already provide privacy bounds that improve over those from lossless accountants (see Section~\ref{sec:numerics}).}
\begin{corollary}[Privacy guarantee from \eqref{eq:1dopt}]\label{cor:1dopt} For any $\alpha > 1$ and $\delta \in (0, 1)$, DP-SGD (Definition~\ref{def:dp-sgd}) with $T$ iterations is $(\eps, \delta)$-DP with
\begin{equation}\label{eq:eps}
    \eps = T\rho\left(\alpha, p, \frac{2}{\sigma}\right) + \log \left(1 - \frac{1}{\alpha}\right) - \frac{1}{\alpha-1}\log (\delta \alpha)\,,
\end{equation}
where $\rho$ satisfies \eqref{eq:1dopt}.
\end{corollary}

\subsection{Theoretical results for implementation}
The problem \eqref{eq:1dopt} over $u \in [0, 1/2]$ is parametrized by $\alpha, p$ and $\Delta$ (equivalently $\sigma$). We explore simple and sufficient conditions for these hyperparameters, under which the solutions of the  problem \eqref{eq:1dopt} are explicitly determined. For numerical implementation with hyperparameters outside of those conditions, we use a branch-and-bound approach to iteratively update an upper bound on \eqref{eq:1dopt}---see Appendix~\ref{subsec:numerics-implementation} for details.

\begin{proposition}[Sufficient conditions for characterizing \eqref{eq:1dopt}]\label{prop:sufficient} Consider the problem \eqref{eq:1dopt} over $u \in [0, 1/2]$.
\begin{enumerate}
    \item[(a)] (Large $\alpha$) \eqref{eq:1dopt} is maximized at $u = 0$ if
    \[\alpha \geq 2 \vee \left(1 + \frac{\sigma^2 \log(2/p)}{2(1-p)}\right)\,.
    \]
    \item[(b)] (Small $p$) \eqref{eq:1dopt} is maximized at either $u = 0$ or $u = 1/2$ if $\sigma^2 > \alpha$ and
    \[p \leq \left(1 + \left(3 + \frac{(12\sqrt{2} + 8)\alpha + 8\sqrt{2}-8}{\sigma^2 - \alpha}\right)e^{\frac{2}{\sigma^2} + \frac{6}{\sigma^2 - \alpha}}\right)^{-1}\,.
    \]
\end{enumerate}
\end{proposition}
The proof is deferred to Appendix~\ref{subsec:sufficient}. 
Another question lies in choosing the R\'enyi order $\alpha$ for Corollary~\ref{cor:1dopt}.
We show that for the final converted $\eps$, the curve $\alpha \mapsto \eps(\alpha)$ is unimodal in $\alpha$. This implies, for example, that calculating \eqref{eq:eps} over integers $2 \leq \alpha \leq \alpha_{\max}$ only requires evaluation at $O(\log \alpha_{\max})$ orders, based on binary search (for sufficiently large $\alpha_{\max}$).

\begin{proposition}[Unimodality in R\'enyi order]\label{prop:unimodal} For any $T, p, \sigma$ and $\delta\in (0, 1)$, $\eps$ in \eqref{eq:eps} as a function of $\alpha \in (1, \infty)$ is unimodal, i.e., minimized at a unique point.
\end{proposition}
The proof is deferred to Appendix~\ref{subsec:unimodal}.

\section{Numerical comparisons}\label{sec:numerics}

We numerically compare the $(\eps, \delta)$-DP guarantees from three different R\'enyi accountants (ours from Corollary~\ref{cor:1dopt}, \citep{birrell2024rdp}, \citep{wang2021rdp}) and lossless accountants (Proposition~\ref{prop:privacy-lossless}), for different values of the DP-SGD hyperparameters $p$ (subsampling rate), $\sigma$ (noise multiplier), and $T$ (number of iterations). 
Our code for reproducing the numerics can be found at: \url{https://github.com/jinhobok/renyi_vs_lossless_dpsgd}.
\begin{figure}[t]
    \centering

    \begin{minipage}{\textwidth}
        \includegraphics[width=0.45\textwidth]{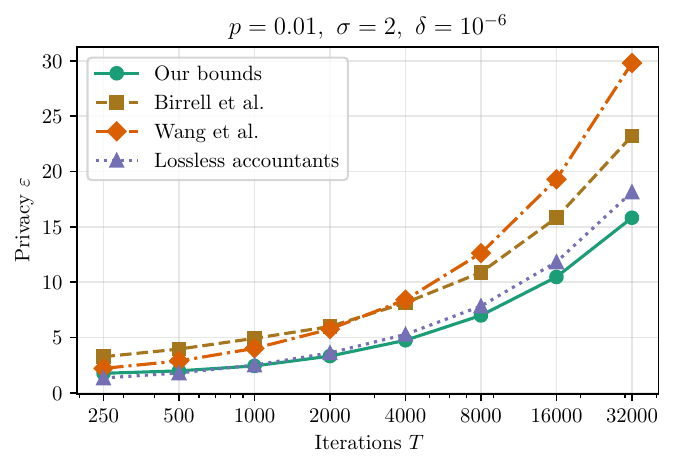}
        \hfill
        \includegraphics[width=0.45\textwidth]{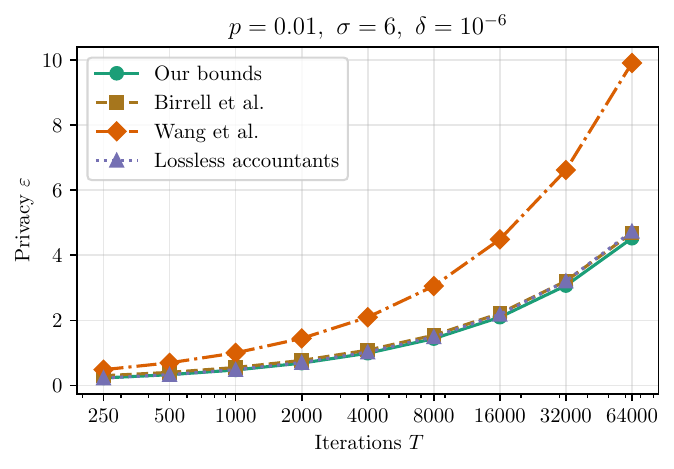}
        \caption{Accounted privacy $\eps$ for noise multipliers $\sigma \in \{2, 6\}$. Our bounds improve over existing bounds as the number of iterations increases.}
        \label{fig:varyingsigma}
    \end{minipage}

    \begin{minipage}{\textwidth}
        \includegraphics[width=0.45\textwidth]{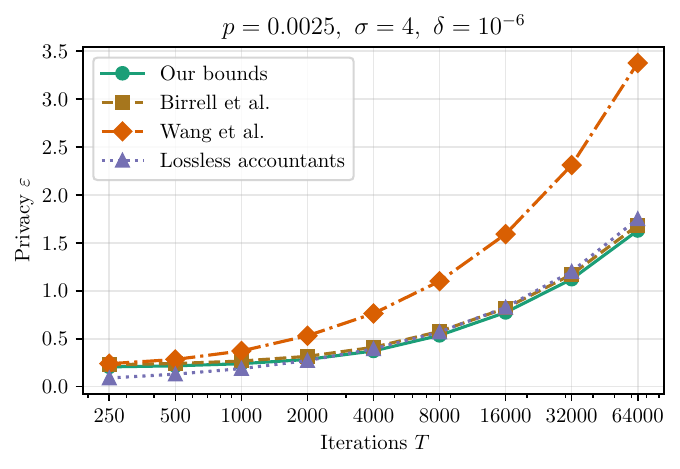}
        \hfill
        \includegraphics[width=0.45\textwidth]{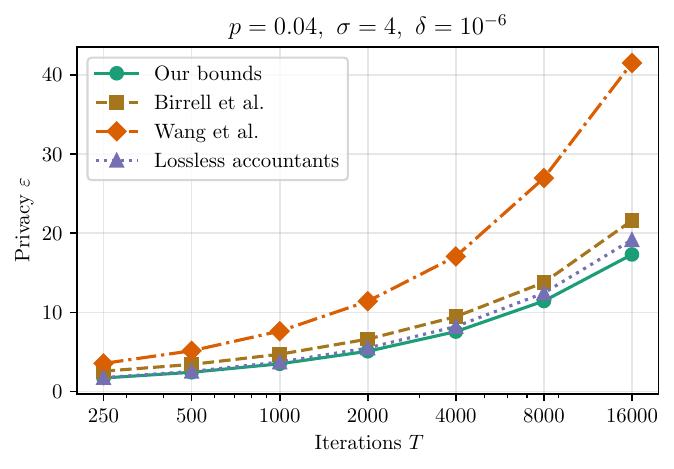}
        \hfill
        \caption{Accounted privacy $\eps$ for subsampling rates $p \in \{0.0025, 0.04\}$. The improvements from our bounds over existing bounds are more prominent for larger $p$.}
        \label{fig:varyingp}
    \end{minipage}
\end{figure}

Figure~\ref{fig:varyingsigma} displays numerical values of $\eps$ for different values of $\sigma$, whereas Figure~\ref{fig:varyingp} displays those for different values of $p$. Throughout, it can be observed that our privacy bounds always improve over those from existing R\'enyi accountants. Furthermore, while our bounds are sometimes larger than those from lossless accountants (for small $T$ here; see also Appendix~\ref{subsec:numerics-full} for other cases), they improve over those as the number of iterations increases. This is consistent with the original motivation of this work on using different dominating pairs: compared to R\'enyi accountants, lossless accountants use another dominating pair whose corresponding privacy guarantee becomes relatively loose as the number of compositions increases (see Section~\ref{sec:dominatingpairs}). We observe similar findings over broader ranges of the parameters; the full results can be found in Appendix~\ref{subsec:numerics-full}.

\section{Asymptotic comparisons}\label{sec:asymp}
In addition to the numerical comparisons, we compare different privacy accounting methods for DP-SGD based on their asymptotics as the subsampling rate $p \to 0$. This simplification helps us provide a theoretical explanation about how our privacy guarantees may improve over those from existing accountants.

\subsection{Comparison with lossless accountants} 
Here, we compare our bounds with those provided by lossless accountants in the asymptotic regime $T \to \infty, p \to 0$ and $Tp^2 \to \nu \in (0, \infty)$. We focus on this regime because the privacy guarantee from the lossless accountants converges to that of the Gaussian mechanism; this is the celebrated central limit theorem, originally formulated in $f$-DP \citep[Corollary 4]{drs22}. For notational simplicity, throughout this subsection we fix $\eps$ and let $\delta$ denote the limiting value (limit superior if the limit is not guaranteed to exist) of the corresponding privacy parameter.

\begin{lemma}[Limit of $\delta$ for lossless accountants]\label{lem:clt} If $T \to \infty, p \to 0$ and $Tp^2 \to \nu$, the limiting privacy parameter $\delta = \delta_{\mathrm{L}}(\eps)$ for lossless accountants satisfies 
\[\delta_{\mathrm{L}}(\eps) = \Phi\left(-\frac{\eps}{\mu} + \frac{\mu}{2}\right) - e^{\eps}\Phi\left(-\frac{\eps}{\mu} - \frac{\mu}{2}\right)\,,
\]
where $\mu \triangleq \sqrt{2\nu(e^{4/\sigma^2}\Phi(3/\sigma) + 3\Phi(-1/\sigma) - 2)}$.
\end{lemma}

The following is a simple upper bound for the limiting $\delta$ for our accountant, in the same asymptotic regime. As a caveat, we note that this result is based on a R\'enyi DP bound with R\'enyi order that may not be an integer.
\begin{proposition}[Limit of $\delta$ for our accountant]\label{prop:lim} If $T \to \infty, p \to 0$ and $Tp^2 \to \nu$, the limiting privacy parameter $\delta = \delta(\eps)$ for our accountant satisfies
\[\delta(\eps) \leq \frac{\lambda }{\lambda +\eps}\left(\frac{\eps}{\lambda + \eps}\right)^{\eps / \lambda }\exp \left(-\frac{\eps^2}{2\lambda} + \frac{\eps}{2}\right)\,,
\]
where $\lambda \triangleq 2\nu(e^{4/\sigma^2}-e^{2/\sigma^2})$.
\end{proposition}
\begin{proof}
    For each $\alpha$, from Proposition~\ref{prop:pexpansion} the limiting R\'enyi DP bound is given as $\nu \alpha(e^{4/\sigma^2}-e^{2/\sigma^2})$. Plugging this in Corollary~\ref{cor:1dopt} and rearranging, we obtain
    \begin{equation}\label{eq:asymprenyi-delta}
        \delta(\eps) \leq \frac{1}{\alpha}\left(1-\frac{1}{\alpha}\right)^{\alpha-1}\exp((\alpha-1)(\nu \alpha(e^{4/\sigma^2}-e^{2/\sigma^2}) - \eps))\,.
    \end{equation}
    The conclusion follows with $\alpha = 1+\eps/\lambda$ in \eqref{eq:asymprenyi-delta}.\footnote{This choice of $\alpha$ is loosely based on first-order optimality: %
    it is straightforward to check that the derivative of \eqref{eq:asymprenyi-delta} being $0$ is equivalent to $\lambda \alpha - (\eps +\lambda/2) - \log(\alpha/(\alpha-1)) = 0$. For $\alpha = 1 + \eps / \lambda$, this last line is equal to $\lambda/2 - \log(1+\lambda/\eps)$ which is close to $0$ for small $\lambda$.}
\end{proof}

\begin{figure}[t]
    \centering
    \includegraphics[width=0.7\textwidth]{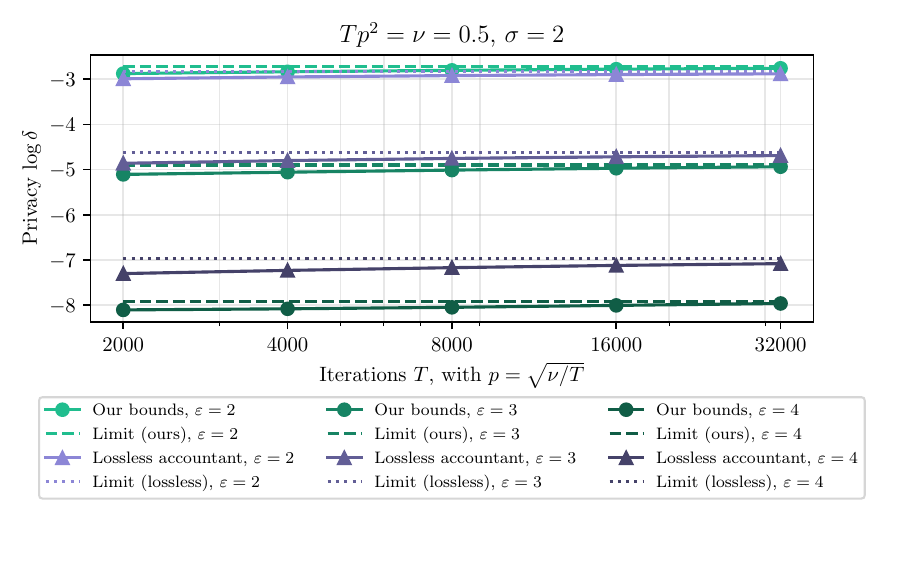}
    \caption{Illustration of privacy bounds for finite $T$ and their asymptotic limit (Lemma~\ref{lem:clt} and Proposition~\ref{prop:lim}; note that the latter is a limit superior). Here, $p$ depends on $T$, keeping $Tp^2 = \nu$. As $\eps$ increases, our bounds provide better privacy guarantees than those from lossless accountants. The privacy bounds become closer to their limits as $T$ increases.}
    \label{fig:asymp-renyi-lossless}
\end{figure}

Figure~\ref{fig:asymp-renyi-lossless} illustrates the privacy parameter $\delta$ (in log scale) for $(\eps, \delta)$-DP obtained from finite $T$, as well as their limiting values. It can be observed from the plot that as $\eps$ increases, our bounds yield smaller $\delta$ than those from lossless accountants, both for finite $T$ and the limits. In fact, the following result can be shown:
\begin{corollary}[Comparison between limiting $\delta$]\label{cor:lim} In the setting of Lemma~\ref{lem:clt} and Proposition~\ref{prop:lim}, $\mu^2 > \lambda$ and 
\[\log \delta(\eps) - \log \delta_{\mathrm{L}}(\eps) \leq  - \frac{1}{2}\left(\frac{1}{\lambda} - \frac{1}{\mu^2}\right)\eps^2 + \log \eps + O_{\mu, \lambda}(1)\,.
\]
\end{corollary}
See Appendix~\ref{subsec:lim} for its proof. Here, the right hand side eventually decreases to $-\infty$ as $\eps$ increases, illustrating how our privacy bounds may improve over those from lossless accountants.

\subsection{Comparison between R\'enyi accountants}

In this subsection, we provide a simple comparison between R\'enyi accountants, focusing on their one-step R\'enyi DP bounds. Namely, we consider our bound (Theorem~\ref{thm:1dopt}) and bounds from \citep{wang2021rdp} and \citep{birrell2024rdp}, comparing their 
asymptotic expansions in the subsampling rate $p$. 

As highlighted in \citep[Section 3.4]{birrell2024rdp}, the one-step R\'enyi DP bound $\rho_{\mathrm{W}}$ from \citep[Theorem 9]{wang2021rdp} satisfies $\rho_{\mathrm{W}} = 2\alpha(e^{4/\sigma^2}-1)p^2 + O_{\alpha, \sigma}(p^3)$,
whereas the one-step R\'enyi bound $\rho_{\mathrm{B}}$ from \citep[Theorem 3.4]{birrell2024rdp} satisfies $\rho_{\mathrm{B}} = \alpha (e^{4/\sigma^2}-e^{2/\sigma^2})p^2  + O_{\alpha, \sigma}(p^3)$.
Thus, \cite{birrell2024rdp} provides an improvement by a factor of roughly $4$ (leading term of $8\alpha p^2/\sigma^2$ versus $2\alpha p^2 /\sigma^2$, for large $\sigma$) over \cite{wang2021rdp} in R\'enyi DP.

In the following, we present the asymptotic expansion of our bound \eqref{eq:1dopt} in $p$. Notably, while its leading order term $\alpha (e^{4/\sigma^2}-e^{2/\sigma^2})p^2$ is exactly the same as that in \citep{birrell2024rdp}, its next order term (cubic in $p$) is strictly smaller. We remark that the bound from \citep{birrell2024rdp} has a tuning order $m$ that controls the degree of the Taylor expansion in their result, and $m \geq 4$ is explicitly suggested.

\begin{proposition}\label{prop:pexpansion} The one-step R\'enyi DP bound $\rho = \rho(\alpha, p, 2/\sigma)$ from \eqref{eq:1dopt} satisfies
\begin{equation}\label{eq:1dopt-pexpansion}
\begin{aligned}
    \rho
    &= \alpha (e^{\frac{4}{\sigma^2}} - e^{\frac{2}{\sigma^2}})p^2 -\frac{\alpha}{2}(e^{\frac{12}{\sigma^2}}-e^{\frac{8}{\sigma^2}} - 2e^{\frac{4}{\sigma^2}} + 2e^{\frac{2}{\sigma^2}})p^3 + O_{\alpha, \sigma}(p^4)\,,
\end{aligned}
\end{equation}
whereas the one-step R\'enyi DP bound $\rho_B$ from \cite{birrell2024rdp} (for any tuning order $m \geq 4$) satisfies
\begin{equation}
\begin{aligned}
    \rho_B &= \alpha (e^{\frac{4}{\sigma^2}} - e^{\frac{2}{\sigma^2}})p^2 + \frac{\alpha(\alpha + 2)}{2}\sqrt{(e^{\frac{4}{\sigma^2}} - 1)(e^{\frac{24}{\sigma^2}} - 4e^{\frac{12}{\sigma^2}}+6e^{\frac{4}{\sigma^2}}-3)}p^3 + O_{\alpha, \sigma}(p^4)\,.
\end{aligned}
\end{equation}
\end{proposition}
See Appendix~\ref{subsec:pexpansion} for its proof. It can be checked that for our bound, the third-order term is strictly negative, 
whereas the bound from \citep{birrell2024rdp} has a strictly positive third-order term. This is consistent with the difference between the plots in Figure~\ref{fig:varyingp}: for larger $p$, the bound from \citep{birrell2024rdp} is further from our bound, as the difference between the $p^3$ terms becomes more prominent. Figure~\ref{fig:asymp-renyi-renyi} illustrates the asymptotic expansions (up to $p^3$ terms) as well as their limitations, where their quality of approximation depends on other hyperparameters.

\begin{figure}[t]
    \centering
    \includegraphics[width=0.7\textwidth]{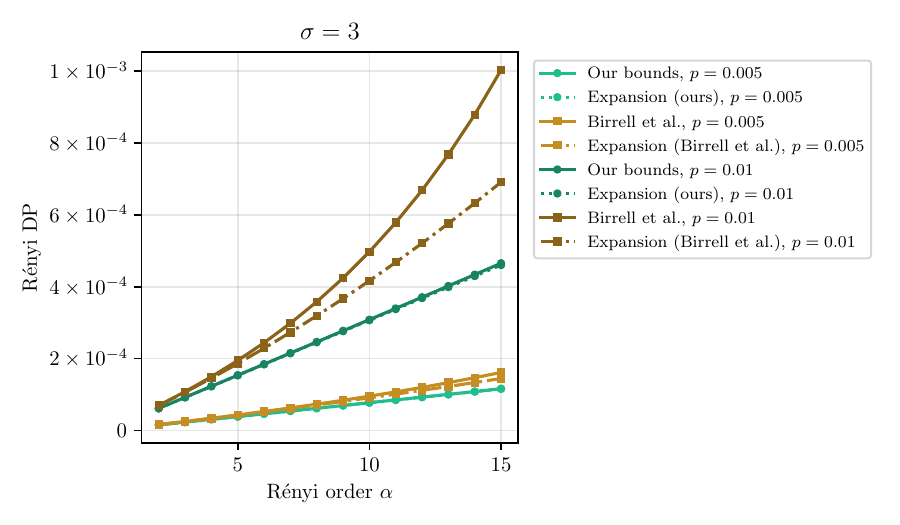}
    \caption{Illustration of R\'enyi DP bounds and their asymptotic expansions (Proposition~\ref{prop:pexpansion}, truncated at $p^3$ terms). For \citep{birrell2024rdp}, we set the tuning parameter to be $m = 4$. The actual bounds are close to the asymptotic expansions when the R\'enyi order $\alpha$ is small, but they deviate more as $\alpha$ increases, in particular for the bounds from \citep{birrell2024rdp}.}
    \label{fig:asymp-renyi-renyi}
\end{figure}

\section{Discussion}\label{sec:discussion}

In this work, we studied R\'enyi accountants and lossless accountants for DP-SGD with fixed-size batches under replace adjacency. In particular, we characterized the privacy guarantee of a R\'enyi accountant as the value of a one-dimensional problem, and provided numerical illustrations as well as theoretical explanations on how it can improve over privacy guarantees from lossless accountants throughout different hyperparameter settings.

An interesting follow-up question is to study whether the notion of hockey-stick dominating pair (Definition~\ref{def:dominatingpairs}) can be relaxed, e.g., for inequalities in $\gamma$ holding for a subset of $[0, \infty)$. While such a modification is known to be invalid if done ``na\"ively'' \citep[Section 7]{lebeda2025pitfalls}, it may be possible that there are other valid approaches, especially for specific mechanisms. More generally, one can ask what types of divergences can be used for privacy accounting---e.g., for handling data processing inequality and composition \citep{mpst21, cruz2026complete}. At a fundamental level, we believe that understanding the powers and limits of different privacy accountants is an important question for both theorists and practitioners of differential privacy.

\section*{AI use}
We used OpenAI's GPT-5.6 to assist with generating new content for this paper, including formulating mathematical claims, writing code for numerical experiments, and proofreading. We reviewed and verified all AI-assisted content, and all writing in the main text and appendix was prepared by the authors who take full responsibility for the content of the paper.

\addcontentsline{toc}{section}{References}
\bibliographystyle{alpha}
\bibliography{ref}

\newpage
\appendix
\section{Deferred details for proofs}\label{sec:proofs}

\subsection{Proof of Theorem~\ref{thm:1dopt}}\label{subsec:1dopt}
The proof proceeds with successive steps of simplifying the problem \eqref{eq:isoproblem}.

    \paragraph*{Reduction with respect to Loewner order.} We first claim that it suffices to consider positive semidefinite matrices that are maximal in the Loewner order. For that, let $a', b' \in \reals^d$ be such that
    \[\begin{pmatrix}
        \|a\|^2 & \langle a, b \rangle \\ \langle a, b \rangle & \|b\|^2
    \end{pmatrix} \preceq \begin{pmatrix}
        \|a'\|^2 & \langle a', b' \rangle \\ \langle a', b' \rangle & \|b'\|^2
    \end{pmatrix}\,,
    \]
    or equivalently, $\|sa + tb\|^2 \leq \|sa' + tb'\|^2$ for all $s, t \in \reals$. Consider a linear map $A: \reals^d \to \reals^d$ such that
    \begin{align*}
        A(sa' + tb') &= sa + tb \text{ for all } s, t \in \reals, \\
        \quad A(u') &= 0 \text{ for all } u' \in \mathrm{span}\{a', b'\}^{\perp}\,.
    \end{align*}
    Then $A$ is well-defined, since $sa' + tb' = 0$ implies $sa + tb = 0$. Furthermore, $A^T A \preceq I_d$ because for any $x \in \reals^d$, writing as $x = sa' + tb' + x^\perp$ with $x^\perp \in \mathrm{span}\{a', b'\}^{\perp}$ we have
    \[x^T A^T A x = \|sa + tb\|^2 \leq \|sa' + tb'\|^2 \leq \|sa' + tb' + x^\perp\|^2 = \|x\|^2\,.
    \]
    Now consider a (randomized) post-processing map $z \mapsto  Az + \calN(0, I_d - AA^T)$ (note that $I_d - AA^T \succeq 0$, from $A^TA \preceq I_d$). Then from its definition, $a'\Ber(p) + \calN(0, I_d)$ and $b'\Ber(p) + \calN(0, I_d)$ are respectively mapped by this post-processing to $a\Ber(p) + \calN(0, I_d)$ and $b\Ber(p) + \calN(0, I_d)$. By data processing inequality,
    \[D_\alpha(a\Ber(p) + \calN(0, I_d) \| b\Ber(p) + \calN(0, I_d)) \leq  D_\alpha(a'\Ber(p) + \calN(0, I_d) \| b'\Ber(p) + \calN(0, I_d))
    \]
    \paragraph*{Reduction to $d = 2$.} 
    For each fixed $a, b$, maximizing the objective in \eqref{eq:isoproblem} is equivalent to maximizing   \begin{equation}\label{eq:2dgaussianintegral}
        \exp((\alpha-1)\rho(\alpha, p, \Delta)) =\E[(1-p + pe^{\langle a, Z \rangle - \|a\|^2 / 2})^{\alpha}(1-p + pe^{\langle b, Z \rangle - \|b\|^2 / 2})^{1-\alpha}]\,,
    \end{equation}
    where $Z \sim \calN(0, I_d)$. In particular, \eqref{eq:2dgaussianintegral} only depends on the joint distribution $(\langle a, Z \rangle, \langle b, Z \rangle) \sim \calN\left(0, \begin{pmatrix}
        \|a\|^2 & \langle a, b \rangle \\ \langle a, b \rangle & \|b\|^2
    \end{pmatrix}\right)$. By reparametrizing the covariance matrix as $\Sigma \triangleq \begin{pmatrix}
        \|a\|^2 & \langle a, b \rangle \\ \langle a, b \rangle & \|b\|^2
    \end{pmatrix}$, it suffices to maximize \eqref{eq:2dgaussianintegral} over all $2 \times 2$ positive semidefinite matrices $\Sigma$.

    \paragraph*{Further reduction within $2 \times 2$ covariances.}
    The constraints $\|a\| \leq \Delta, \|b\| \leq \Delta, \|a - b\|\leq \Delta$ are together equivalent to 
    \[\Sigma_{11} \leq \Delta^2, \Sigma_{22} \leq \Delta^2, \Sigma_{11} + \Sigma_{22} - 2\Sigma_{12} \leq \Delta^2.
    \]
    First, it suffices to consider $\Sigma$ such that at least two of the inequalities are equalities. Indeed, if $\Sigma_{22} < \Delta^2$ and $\Sigma_{11} + \Sigma_{22} - 2\Sigma_{12} < \Delta^2$, one can consider $\Sigma + \eps\begin{pmatrix}
        0 & 0 \\ 0 & 1
    \end{pmatrix}$ for a sufficiently small $\eps > 0$, which dominates $\Sigma$ in the Loewner order and still satisfies the constraints; a similar argument applies if $\Sigma_{11} < \Delta^2$ and $\Sigma_{11} + \Sigma_{22} - 2\Sigma_{12} < \Delta^2$. If $\Sigma_{11} < \Delta^2$ and $\Sigma_{22} < \Delta^2$, one can consider $\Sigma + \eps\begin{pmatrix}
        1 & 1 \\ 1 & 1
    \end{pmatrix}$.

    Thus, we have three (possibly overlapping) cases:
    \begin{align}
        \Sigma_{11} &\leq  \Delta^2, \Sigma_{22} = \Delta^2, \Sigma_{11} + \Sigma_{22} - 2\Sigma_{12} = \Delta^2 \label{eq:1dopt-case1}\\
        \Sigma_{11} &= \Delta^2, \Sigma_{22} \leq \Delta^2, \Sigma_{11} + \Sigma_{22} - 2\Sigma_{12} = \Delta^2 \label{eq:1dopt-case2}\\
        \Sigma_{11} &= \Delta^2, \Sigma_{22} = \Delta^2, \Sigma_{11} + \Sigma_{22} - 2\Sigma_{12} \leq \Delta^2\,. \label{eq:1dopt-case3}
    \end{align}
    The third case \eqref{eq:1dopt-case3} is redundant: the partial derivative of \eqref{eq:2dgaussianintegral} with respect to $\Sigma_{12}$ is equal to
    \[\alpha(1-\alpha)p^2\E[e^{\langle a, Z \rangle - \|a\|^2/2}(1-p + pe^{\langle a, Z \rangle - \|a\|^2 / 2})^{\alpha-1}\times e^{\langle b, Z \rangle - \|b\|^2/2}(1-p + pe^{\langle b, Z \rangle - \|b\|^2 / 2})^{-\alpha}]\,,
    \]
    which is negative. Combined with $\Sigma_{12} \geq \Delta^2/2$ from \eqref{eq:1dopt-case3}, it follows that \eqref{eq:2dgaussianintegral} is maximized at $\Sigma_{12} = \Delta^2/2$ in this case. This implies $\Sigma_{11} + \Sigma_{22} - 2\Sigma_{12} = \Delta^2$, and hence the case \eqref{eq:1dopt-case3} subsumed in the first two cases \eqref{eq:1dopt-case1} and \eqref{eq:1dopt-case2}.

    For \eqref{eq:1dopt-case1}, without loss of generality we can parametrize as (with $d = 2$) $b = \Delta(1, 0)$ and $a = \Delta(u, \sqrt{2u - u^2})$ for some $u \in [0, 1/2]$. Thus the maximum of \eqref{eq:2dgaussianintegral} subject to \eqref{eq:1dopt-case1} is equal to
    \[\max_{u \in [0, 1/2]}D_\alpha\left(\Delta(u, \sqrt{2u - u^2}) \times \mathrm{Ber}(p) + \calN(0, I_2) 
\Big \|\Delta(1, 0)\times \mathrm{Ber}(p) + \calN(0, I_2)\right)\,.
    \]
    Similarly, for \eqref{eq:1dopt-case2} we can let $a = \Delta(1, 0)$ and $b = \Delta(u, \sqrt{2u - u^2})$; the maximum of \eqref{eq:2dgaussianintegral} subject to \eqref{eq:1dopt-case2} is then equal to
    \[\max_{u \in [0, 1/2]}D_\alpha\left(\Delta(1, 0)\times \mathrm{Ber}(p) + \calN(0, I_2)     \Big \| \Delta(u, \sqrt{2u - u^2}) \times \mathrm{Ber}(p) + \calN(0, I_2)\right)\,.
    \]
    We show that the latter is always at least as large as the former, i.e., it only suffices to consider the second case \eqref{eq:1dopt-case2}.
    \begin{lemma}\label{lem:renyiasymm} For any $u \in [0, 1/2]$,
    \begin{align*}
        &D_\alpha\Big(\Delta(1, 0)\times \mathrm{Ber}(p) + \calN(0, I_2) \Big \| \Delta(u, \sqrt{2u - u^2}) \times \mathrm{Ber}(p) + \calN(0, I_2)\Big) \\
        &\geq D_\alpha\Big(\Delta(u, \sqrt{2u - u^2}) \times \mathrm{Ber}(p) + \calN(0, I_2)  \Big \|\Delta(1, 0)\times \mathrm{Ber}(p) + \calN(0, I_2)\Big)\,.
    \end{align*}
    \end{lemma}
See Appendix~\ref{subsubsec:renyiasymm} for its proof. This completes the proof of the theorem.

\subsubsection{Proof of Lemma~\ref{lem:renyiasymm}}\label{subsubsec:renyiasymm}

Since $\int \left(\frac{\diff P}{\diff Q}(\omega)\right)^\alpha \diff Q(\omega) = \exp((\alpha-1)D_{\alpha}(P \| Q))$, it suffices to compare the integrals. Let $a = \Delta(1, 0)$ and $b = \Delta(u, \sqrt{2u - u^2})$ (with $u \in [0, 1/2]$); for this subsection only, we write $\phi(x)$ to denote specifically the density of $\calN(0, I_2)$ at $x \in \reals^2$. Furthermore, let $\ell$ be the perpendicular bisector of the line segment $(a, b)$. Then there are two halfspaces determined by $\ell$, namely:
\begin{align*}
    H^+ &\triangleq \left\{x \in \reals^2: \left\langle b-a, x - \frac{a+b}{2}\right\rangle \leq 0 \right\}, \\
    H^- &\triangleq \left\{x \in \reals^2: \left\langle b-a, x - \frac{a+b}{2}\right\rangle \geq 0 \right\}\,.
\end{align*}
Finally, define $R: \reals^2 \to \reals^2$ to be the reflection across $\ell$. Then the integral corresponding to $D_\alpha (\Delta(1, 0)\times \mathrm{Ber}(p) + \calN(0, I_2) \| \Delta(u, \sqrt{2u - u^2}) \times \mathrm{Ber}(p) + \calN(0, I_2))$ is
\begin{align*}
    &\int_{\reals^2} ((1-p)\phi(x) + p\phi(x - a))^{\alpha}((1-p)\phi(x) + p\phi(x-b))^{1-\alpha} \diff x \\
    &= \int_{H^+} ((1-p)\phi(x) + p\phi(x - a))^{\alpha}((1-p)\phi(x) + p\phi(x-b))^{1-\alpha} \diff x \\
    &\quad + \int_{H^-} ((1-p)\phi(y) + p\phi(y - a))^{\alpha}((1-p)\phi(y) + p\phi(y-b))^{1-\alpha} \diff y \\
    &= \int_{H^+} ((1-p)\phi(x) + p\phi(x - a))^{\alpha}((1-p)\phi(x) + p\phi(x-b))^{1-\alpha} \diff x \\
    &\quad + \int_{H^+} ((1-p)\phi(R(x)) + p\phi(x - b))^{\alpha}((1-p)\phi(R(x)) + p\phi(x-a))^{1-\alpha} \diff x\,,
\end{align*}
where the second equality holds as $R$ is the reflection across $\ell$ ($y = R(x) \Rightarrow R(a) = b, R(b) = a$). Similarly, the integral corresponding to $D_\alpha (\Delta(u, \sqrt{2u - u^2}) \times \mathrm{Ber}(p) + \calN(0, I_2) \| \Delta(1, 0)\times \mathrm{Ber}(p) + \calN(0, I_2))$ is
\begin{align*}
    &\int_{\reals^2} ((1-p)\phi(x) + p\phi(x - b))^{\alpha}((1-p)\phi(x) + p\phi(x-a))^{1-\alpha} \diff x \\
    &= \int_{H^+} ((1-p)\phi(x) + p\phi(x - b))^{\alpha}((1-p)\phi(x) + p\phi(x-a))^{1-\alpha} \diff x \\
    &\quad + \int_{H^-} ((1-p)\phi(y) + p\phi(y - b))^{\alpha}((1-p)\phi(y) + p\phi(y-a))^{1-\alpha} \diff y \\
    &= \int_{H^+} ((1-p)\phi(x) + p\phi(x - b))^{\alpha}((1-p)\phi(x) + p\phi(x-a))^{1-\alpha} \diff x \\
    &\quad + \int_{H^+} ((1-p)\phi(R(x)) + p\phi(x - a))^{\alpha}((1-p)\phi(R(x)) + p\phi(x-b))^{1-\alpha} \diff x\,.    
\end{align*}
Then their difference can be written as
\begin{equation}\label{eq:renyiasymm-calculus}
\begin{aligned}
    &\int_{\reals^2} ((1-p)\phi(x) + p\phi(x - a))^{\alpha}((1-p)\phi(x) + p\phi(x-b))^{1-\alpha} \diff x \\
    &\quad - \int_{\reals^2} ((1-p)\phi(x) + p\phi(x - b))^{\alpha}((1-p)\phi(x) + p\phi(x-a))^{1-\alpha} \diff x \\
    &= \int_{H^+} \left(f((1-p)\phi(x), p\phi(x-a), p\phi(x-b)) - f((1-p)\phi(R(x)), p\phi(x-a), p\phi(x-b))\right) \diff x\,,
\end{aligned}
\end{equation}
where $f(u, v, w) \triangleq (u + v)^\alpha(u+w)^{1-\alpha} - (u+w)^\alpha(u+v)^{1-\alpha}$ for $u, v, w > 0$. 

We claim that for each fixed $v, w$ with $v \geq w$, $u \mapsto f(u, v, w)$ is decreasing. To see this, by writing $t \triangleq (u+v) / (u+w) \geq 1$ we have
\[f(u, v, w) = (u+w)(t^\alpha - t^{1-\alpha}) = (v-w)\frac{t^\alpha - t^{1-\alpha}}{t-1}\,,
\]
where the function $g(t) \triangleq t^\alpha - t^{1-\alpha}$ is convex for $t \geq 1$. Thus, $t \mapsto (t^{\alpha} - t^{1-\alpha})/(t-1) = (g(t) - g(1))/(t-1)$ is increasing in $t$. Since $t$ is decreasing in $u$, $f$ is decreasing in $u$, proving the claim.

Going back to \eqref{eq:renyiasymm-calculus}, it is straightforward to check that for $x \in H^+$, $\|x - a\| \leq \|x - b\|$ (since $a \in H^+$) and $\|R(x)\| \leq \|x\|$ (since $0 \in H^-$, from $\|a\|^2 \geq \|b\|^2$). These respectively imply $\phi(x-a) \geq \phi(x-b)$  and $(1-p)\phi(x) \leq (1-p)\phi(R(x))$. Combining these with the aforementioned claim on $f$ (with $u = (1-p)\phi(x)$ or $(1-p)\phi(R(x))$ and $v = p\phi(x-a), w = p\phi(x-b)$), we have
\[f((1-p)\phi(x), p\phi(x-a), p\phi(x-b)) - f((1-p)\phi(R(x)), p\phi(x-a), p\phi(x-b)) \geq 0\,,
\]
implying that \eqref{eq:renyiasymm-calculus} is nonnegative, as desired.

\subsection{Derivative calculations}\label{subsec:derivs}
Here, we record additional notation and details for $F$ and its derivatives that are used throughout our proofs. The calculations in this subsection serve as preliminaries for mathematical arguments in other parts.

The one-step R\'enyi bound $\rho(\alpha, p, \Delta)$ in \eqref{eq:1dopt} can be written as
\[\exp((\alpha-1)\rho(\alpha, p, \Delta)) = \max_{u \in [0, 1/2]}F(u)\,,
\]
where with $a = \Delta(1, 0), b = \Delta(u, \sqrt{2u - u^2})$ and $Z \sim \calN(0, I_2)$, $F$ is defined as
\begin{equation}\label{eq:ffunction}
    F(u) \triangleq \E[(1-p + p\exp(\langle a, Z \rangle - \|a\|^2 / 2))^{\alpha}(1-p + p\exp(\langle b, Z \rangle - \|b\|^2 / 2))^{1-\alpha}]\,.
\end{equation}

First, by letting $h(x, y) \triangleq (1-p + pe^x)^\alpha(1-p + pe^y)^{1-\alpha}$ and $\begin{pmatrix}
    X \\ Y
\end{pmatrix} \sim \calN\left(\Delta^2\begin{pmatrix}
    -1/2 \\ -u
\end{pmatrix}, \Delta^2\begin{pmatrix}
  1 & u \\ u & 2u
\end{pmatrix}\right)$, we have
\[F(u) = \E[h(X, Y)]\,,
\]
where from Gaussian integration by parts,
\begin{equation}\label{eq:firstderiv-pre}
\begin{aligned}
    F'(u) &= -\Delta^2 \E\left[\frac{\partial}{\partial y}h(X, Y)\right] + \Delta^2 \E\left[\frac{\partial^2}{\partial x\partial y}h(X, Y) + \frac{\partial^2}{\partial y^2} h(X, Y)\right] \\
    &= \Delta^2 \E\left[\left(\frac{\partial^2}{\partial y^2} + \frac{\partial^2}{\partial x\partial y} - \frac{\partial}{\partial y}\right)h(X,Y)\right]\,.
\end{aligned}    
\end{equation}
As notation, we define $v(x) \triangleq pe^x / (1-p + pe^x), w(y) \triangleq pe^y/(1-p + pe^y)$. In particular, they satisfy $0 \leq v(x), w(y) \leq 1$ and 
\[\frac{\partial}{\partial x}v(x) = v(x)(1-v(x)), \quad \frac{\partial}{\partial y}w(y) = w(y)(1-w(y))\,.
\]
We will often write $v = v(x)$ and $w = w(y)$ for conciseness. It is straightforward to see that
\begin{align*}
    \frac{\partial}{\partial x}h(x, y) = \alpha vh(x, y), \quad \frac{\partial}{\partial y}h(x, y) = (1-\alpha)wh(x, y)\,,
\end{align*}
and
\begin{gather*}
    \frac{\partial^2}{\partial x^2}h(x, y) = \alpha v(1 + (\alpha - 1)v)h(x, y), \quad \frac{\partial^2}{\partial x\partial y}h(x, y) = \alpha(1-\alpha)vwh(x, y), \\
    \frac{\partial^2}{\partial y^2}h(x, y) = (1-\alpha)w(1-\alpha w)h(x, y)\,.
\end{gather*}
Based on those derivatives, for \eqref{eq:firstderiv-pre} we have
\begin{equation}\label{eq:firstderiv}
    F'(u) = \Delta^2\alpha(\alpha-1)\E[w(Y)(w(Y) - v(X))h(X, Y)]\,.
\end{equation}
Similarly, we have
\begin{align*}
    F''(u) &= \Delta^4\alpha (\alpha - 1)\E\left[\left(\frac{\partial^2}{\partial y^2} + \frac{\partial^2}{\partial x\partial y} - \frac{\partial}{\partial y}\right)w(w-v)h(X, Y)\right] \\
    &= \Delta^4\alpha(\alpha - 1) \E\left[\left(\frac{\partial}{\partial y} + \frac{\partial}{\partial x} - 1\right)(-(\alpha + 1)w^3 + \alpha w^2 v + 2 w^2 - wv)h(X, Y)\right] \\
    &= \Delta^4\alpha(\alpha - 1) \E\bigg[\Big(-3(\alpha+1)w^3(1-w) + 2\alpha w^2(1-w)v + 4w^2(1-w) -w(1-w)v\\
    &\quad + (\alpha-1)\left((\alpha+1)w^4 - \alpha vw^3 - 2w^3 + vw^2\right) + \left(\alpha w^2 v(1-v) - wv(1-v) w^2\right) \\
    &\quad + \alpha v(-(\alpha + 1)w^2  + \alpha w^2v + 2w^2 - wv) + (\alpha + 1)w^3 - \alpha w^2 v - 2w^2 + wv \Big)h(X, Y)\bigg]\,,
\end{align*}
which implies
\begin{equation}\label{eq:secondderiv}
\begin{aligned}
    F''(u) &= \Delta^4 \alpha(\alpha - 1) \E\bigg[\Big((\alpha+1)(\alpha+2)w^4 - 2\alpha(\alpha+1)w^3v + \alpha(\alpha-1)w^2v^2 -4(\alpha+1)w^3 \\
    &\quad+ 5\alpha w^2v - (\alpha-1)wv^2 + 2w^2 - wv\Big)h(X, Y)\bigg]\,.
\end{aligned}
\end{equation}

\subsection{Proof of Proposition~\ref{prop:sufficient}}\label{subsec:sufficient}

Throughout the proof, as in the earlier parts we use $\Delta = 2/\sigma$ for notational convenience. As noted in Appendix~\ref{subsec:derivs}, we characterize the problem \eqref{eq:1dopt} in terms of the function $F$ \eqref{eq:ffunction} (see also the notations and calculations there for reference).
\begin{enumerate}
    \item[(a)] We show that \eqref{eq:ffunction} (hence \eqref{eq:1dopt}) is maximized at $u = 0$ under a slightly more general condition, that is,
    \begin{equation}\label{eq:optunique-largealpha}
        p\exp\left(\frac{(\alpha-1)(1-p)\Delta^2}{2}\right) \geq \frac{\alpha}{\alpha - 1}\,.
    \end{equation}
    Indeed, if $ \alpha \geq 2$ then \eqref{eq:optunique-largealpha} is implied by $p\exp((\alpha-1)(1-p)\Delta^2 / 2) \geq 2 \Leftrightarrow \alpha \geq 1 + 2\log(2/p)/(\Delta^2(1-p))$.
    For \eqref{eq:ffunction}, we have
    \begin{align*}
        F(u) &= \E[(1-p + pe^{\langle a, Z \rangle - \|a\|^2/2})^{\alpha}(1-p + pe^{\langle b, Z \rangle - \|b\|^2/2})^{1-\alpha}] \\
        &\leq \E[(1-p + pe^{\langle a, Z \rangle - \|a\|^2/2})^{\alpha}(e^{p\langle b, Z \rangle - p\|b\|^2/2})^{1-\alpha}] \\
        &= e^{(\alpha-1)p\Delta^2u}\E[(1-p + pe^{\langle a, Z \rangle - \|a\|^2/2})^{\alpha}e^{-(\alpha-1)p\langle b, Z \rangle}] \\
        &= e^{(\alpha-1)p\Delta^2 u+ (\alpha-1)^2p^2\Delta^2 u}\E[(1-p + pe^{\langle a, Z - (\alpha-1)pb \rangle - \|a\|^2/2})^{\alpha}] \\
        &= e^{(\alpha-1)p\Delta^2 u+ (\alpha-1)^2p^2\Delta^2 u}\E[(1-p + e^{-(\alpha-1)p\Delta^2 u} \times pe^{\langle a, Z\rangle - \|a\|^2/2})^{\alpha}]\,,
    \end{align*}
    where the first inequality is from Jensen's inequality $1-p + pe^x \geq e^{px}$ (and $1-\alpha < 0$); the first equality is from $\|b\|^2 = 2\Delta^2 u$, the second equality is from considering Gaussian shift $Z \to Z + (\alpha-1)pb$, and the last equality is from $\langle a, b \rangle = \Delta^2u$. The last line can be further bounded as
    \begin{align*}
         &e^{(\alpha-1)p\Delta^2 u+ (\alpha-1)^2p^2\Delta^2 u}\E[(1-p + e^{-(\alpha-1)p\Delta^2 u} \times pe^{\langle a, Z\rangle - \|a\|^2/2})^{\alpha}] \\
         &= e^{(\alpha-1)p\Delta^2 u+ (\alpha-1)^2p^2\Delta^2 u}\E\left[\left(e^{-(\alpha-1)p\Delta^2 u}(1 - p + pe^{\langle a, Z\rangle - \|a\|^2/2}) + (1-p)(1-e^{-(\alpha-1)p\Delta^2 u})\right)^{\alpha}\right] \\
         &\leq e^{(\alpha-1)p\Delta^2 u+ (\alpha-1)^2p^2\Delta^2 u}\left(e^{-(\alpha-1)p\Delta^2 u}F(0)^{1/\alpha} + (1-p)(1-e^{-(\alpha-1)p\Delta^2 u})\right)^{\alpha}\,,
    \end{align*}
    by Minkowski's inequality (on $\calL^\alpha$ norm) and $F(0) = \E[(1 - p + pe^{\langle a, Z\rangle - \|a\|^2/2})^{\alpha}]$. Thus, we have
    \begin{equation}\label{eq:optunique-largealpha-deriv}
        \frac{F(u)}{F(0)} \leq e^{(\alpha-1)p\Delta^2 u+ (\alpha-1)^2p^2\Delta^2 u}\left(e^{-(\alpha-1)p\Delta^2 u} + (1-p)F(0)^{-1/\alpha}(1-e^{-(\alpha-1)p\Delta^2 u})\right)^{\alpha}\,.
    \end{equation}
    In particular, the right hand side can be concisely written as a function of $x \triangleq (\alpha-1)p\Delta^2u$:
    \[T(x) \triangleq e^{(1 + (\alpha-1)p)x}\left(e^{-x} + (1-p)F(0)^{-1/\alpha}(1-e^{-x})\right)^{\alpha}\,.
    \]
    From 
    \begin{align*}
        \log T(x) &= (1 + (\alpha-1)p)x + \alpha \log(e^{-x} + (1-p)F(0)^{-1/\alpha}(1-e^{-x})) \\
        &= -(\alpha-1)(1-p)x + \alpha \log(1 + (1-p)F(0)^{-1/\alpha}(e^x - 1))\,,
    \end{align*}
    we have
    \begin{align*}
        T'(x) &= T(x)\left(-(\alpha-1)(1-p) + \frac{\alpha(1-p)F(0)^{-1/\alpha}e^x}{1 + (1-p)F(0)^{-1/\alpha}(e^x-1)}\right) \\
        &\leq T(x)(-(\alpha-1)(1-p) + \alpha(1-p)F(0)^{-1/\alpha}e^x) \\
        &\leq (1-p)T(x)\left(-(\alpha-1) + \frac{\alpha e^x}{pe^{(\alpha-1)\Delta^2/2}} \right)\,,
    \end{align*}
    where the last inequality is from $F(0)^{1/\alpha} \geq \left( \E[(pe^{\langle a, Z \rangle - \|a \|^2 / 2})^\alpha]\right)^{1/\alpha} = pe^{(\alpha-1)\Delta^2/2}$. Since $x = (\alpha-1)p\Delta^2 u \leq (\alpha-1)p\Delta^2/2$, the last line is at most $0$ if
    \[-(\alpha-1) + \frac{\alpha e^{(\alpha-1)p\Delta^2/2}}{pe^{(\alpha-1)\Delta^2/2}} \leq 0 \Leftrightarrow p\exp\left(\frac{(\alpha-1)(1-p)\Delta^2}{2}\right) \geq \frac{\alpha}{\alpha - 1}\,.
    \]
    Thus, \eqref{eq:optunique-largealpha} implies $T'(x) \leq 0$; since $T(0) = 1$, from \eqref{eq:optunique-largealpha-deriv} we have $F(u) / F(0) \leq 1$ for all $u \in [0, 1/2]$, as desired.

    \item[(b)] For $F$ as in \eqref{eq:ffunction}, its first and second derivatives can be written as (see Appendix~\ref{subsec:derivs} and the notations there)
    \begin{align*}
        F'(u) &= \Delta^2\alpha(\alpha-1)\E[w(w-v)h(X, Y)], \\
        F''(u) &= \Delta^4 \alpha(\alpha - 1)\E\bigg[\Big((\alpha+1)(\alpha+2)w^4 - 2\alpha(\alpha+1)w^3v + \alpha(\alpha-1)w^2v^2 -4(\alpha+1)w^3 \\
    &\quad+ 5\alpha w^2v - (\alpha-1)wv^2 + 2w^2 - wv\Big)h(X, Y)\bigg]\,.
    \end{align*}
    We claim that under the stated condition, for $u \in (0, 1/2)$ if $F'(u) = 0$ then $F''(u) > 0$. This indeed implies that the maximum only occurs at the boundary $\{0, 1/2\}$.

    The derivatives can be expressed as sums of expected monomials (in $w, v$) under the tilted measure
    \begin{equation}\label{eq:tiltedmeasure}
    \widetilde{\E}[f(X, Y)] \triangleq \frac{\E[f(X, Y)w(Y)h(X, Y)]}{\E[w(Y)h(X,Y)]} \text{ for }f \geq 0\,. 
    \end{equation}
    Then
    \begin{align*}
        F'(u) = 0 &\Leftrightarrow \widetilde{\E}[w-v] = 0,\\
        F''(u) > 0 &\Leftrightarrow \widetilde{\E}[q(v, w)] > 0\,,
    \end{align*}
    where $q(v, w) \triangleq (\alpha+1)(\alpha+2)w^3 - 2\alpha(\alpha+1)w^2v + \alpha(\alpha-1)wv^2 -4(\alpha+1)w^2 + 5\alpha wv - (\alpha-1)v^2 + 2w - v$. Expanding $q(v, w)$ in descending order (in $v - w$), we have
    \[q(v, w) = (\alpha-1)(\alpha w - 1)(v-w)^2 + (w(3\alpha+2-4\alpha w)-1)(v-w) + w(1-w)(1-2w)\,.
    \]
    Thus
    \begin{equation}\label{eq:polynomialdesc}
        \widetilde{\E}[q(v, w)] = (\alpha-1)\widetilde{\E}[(\alpha w-1)(v-w)^2] + \widetilde{\E}[w(3\alpha + 2-4\alpha w)(v-w)] + \widetilde{\E}[w(1-w)(1-2w)]\,,
    \end{equation}
    where we used $\widetilde{\E}[v-w]=0$ from $F'(u)= 0$.

    The key property that we use for the tilted measure $\widetilde{\Prob}$ is that the density of $(X, Y)$ is strongly log-concave, up to a change of basis.

    \begin{claim}
     For $u \in (0, 1/2)$, if $\alpha \Delta^2/4 < 1$ then the Hessian of the negative log-density of $(X, Y)$ under $\widetilde{\Prob}$ is bounded below (in Loewner order) by $(1-\alpha \Delta^2/4)\left(\Delta^2\begin{pmatrix}
  1 & u \\ u & 2u
\end{pmatrix}\right)^{-1}$.
    \end{claim}
    \begin{proof}
        Under $\widetilde{\Prob}$, $(X, Y)$ has density proportional to $w(y)h(x, y)\psi(x, y) \propto \exp(y + \alpha\log(1-p+pe^x) - \alpha \log(1-p+pe^y))\psi(x, y)$ where $\psi(x, y)$ is the density for $\calN\left(\Delta^2\begin{pmatrix}
    -1/2 \\ -u
\end{pmatrix}, \Delta^2\begin{pmatrix}
  1 & u \\ u & 2u
\end{pmatrix}\right)$. Thus its negative log-density is (up to an additive normalization constant)
\[-y - \alpha\log(1-p + pe^x) + \alpha\log(1-p + pe^y) - \log \psi(x, y)\,,
\]
whose Hessian is
\begin{equation}\label{eq:hessianlb}
\begin{aligned}
    &-\alpha v(x)(1-v(x))\begin{pmatrix}
    1 & 0 \\ 0 & 0
\end{pmatrix} + \alpha w(y)(1-w(y)) \begin{pmatrix}
    0 & 0 \\ 0 & 1
\end{pmatrix} +\left(\Delta^2\begin{pmatrix}
  1 & u \\ u & 2u
\end{pmatrix}\right)^{-1} \\
&\succeq \left(1-\frac{\alpha \Delta^2}{4}\right)\left(\Delta^2\begin{pmatrix}
  1 & u \\ u & 2u
\end{pmatrix}\right)^{-1}\,.
\end{aligned}
\end{equation}

Here, the inequality holds from $v(x)(1-v(x)) \leq 1/4$ and $\begin{pmatrix}
  1 & u \\ u & 2u
\end{pmatrix}^{-1} \succeq \begin{pmatrix}
    1 & 0 \\ 0 & 0
\end{pmatrix}$.
\end{proof}

In particular, this strong log-concavity is useful for bounding the expectations of polynomials in $v, w$.

\begin{lemma}[Polynomial bounds]\label{lem:polynomialbounds} If $\widetilde{\E}[w]= \widetilde{\E}[v]$, then
\begin{align}
    \widetilde{\E}[w^2] &\leq (\widetilde{\E}[w])^2 \exp\left(\frac{\Delta^2}{1-\alpha\Delta^2/4}\right), \label{eq:polynomial1} \\
    \widetilde{\E}[(v-w)^2] &\leq (\widetilde{\E}[w])^2\frac{2\Delta^2}{1-\alpha \Delta^2/4}\exp\left(\frac{\Delta^2}{1-\alpha\Delta^2/4}\right), \label{eq:polynomial2} \\
    |\widetilde{\E}[w(3\alpha + 2-4\alpha w)(v-w)]| &\leq  (\widetilde{\E}[w])^2 \frac{(3\sqrt{2}\alpha + 2\sqrt{2})\Delta^2}{1-\alpha \Delta^2/4}\exp\left(\frac{\Delta^2}{1-\alpha\Delta^2/4}\right)\,. \label{eq:polynomial3}
\end{align}
\end{lemma}
The proof is deferred to Appendix~\ref{subsubsec:polynomialbounds}. By the lemma, for \eqref{eq:polynomialdesc} we have
\begin{align*}
    (\alpha-1)\widetilde{\E}[(\alpha w-1)(v-w)^2] &\geq -(\alpha-1)\widetilde{\E}[(v-w)^2] \geq -(\widetilde{\E}[w])^2\frac{2(\alpha-1)\Delta^2}{1-\alpha \Delta^2/4}\exp\left(\frac{\Delta^2}{1-\alpha\Delta^2/4}\right), \\
    \widetilde{\E}[w(3\alpha + 2-4\alpha w)(v-w)] &\geq -(\widetilde{\E}[w])^2 \frac{(3\sqrt{2}\alpha + 2\sqrt{2})\Delta^2}{1-\alpha \Delta^2/4}\exp\left(\frac{\Delta^2}{1-\alpha\Delta^2/4}\right), \\
    \widetilde{\E}[w(1-w)(1-2w)] &> \widetilde{\E}[w - 3w^2] \geq (\widetilde{\E}[w])\left(1 - 3\widetilde{\E}[w]\exp\left(\frac{\Delta^2}{1-\alpha\Delta^2/4}\right)\right)\,,
\end{align*}
respectively by \eqref{eq:polynomial2}, \eqref{eq:polynomial3}, and \eqref{eq:polynomial1}. Thus,
\begin{equation}\label{eq:qlb}
    \widetilde{\E}[q(v, w)] > (\widetilde{\E}[w])\left(1 - (\widetilde{\E}[w])\exp\left(\frac{\Delta^2}{1-\alpha\Delta^2/4}\right)\left(3 + ((3\sqrt{2} + 2)\alpha + 2\sqrt{2}-2)\frac{\Delta^2}{1-\alpha \Delta^2/4}\right)\right)\,.
\end{equation}
In order to show that the right hand side is nonnegative, it suffices to upper bound $\widetilde{\E}[w]$. First, we have
\[\widetilde{\E}[w] \leq \widetilde{\E}\left[\frac{pe^Y}{1-p}\right] \leq \frac{p}{1-p}\exp\left(\widetilde{\E}[Y] + \frac{u\Delta^2}{1-\alpha \Delta^2/4}\right) \leq \frac{p}{1-p}\exp\left(\widetilde{\E}[Y] + \frac{\Delta^2}{2-\alpha \Delta^2/2}\right)\,,
\]
where the second inequality is from Herbst's argument \citep[Proposition 5.4.1]{bakry2014markov} and the aforementioned strong log-concavity, and the last inequality is from $u \leq 1/2$. For $\widetilde{\E}[Y]$, Stein's lemma yields
\begin{align*}
    &\E\left[\left(\begin{pmatrix}
        X \\ Y
    \end{pmatrix} - \begin{pmatrix}
        -\Delta^2/2 \\ -u\Delta^2
    \end{pmatrix}\right)w(Y)h(X, Y)\right] = \Delta^2\begin{pmatrix}
  1 & u \\ u & 2u
\end{pmatrix} \E[\nabla (w(Y)h(X, Y))] \\
&\Rightarrow \widetilde{\E}\left[\left(\begin{pmatrix}
        X \\ Y
    \end{pmatrix} - \begin{pmatrix}
        -\Delta^2/2 \\ -u\Delta^2
    \end{pmatrix}\right)\right] = \Delta^2\begin{pmatrix}
  1 & u \\ u & 2u
\end{pmatrix}\widetilde{\E}\left[\nabla \log w(Y)h(X, Y) \right] \\
&\hphantom{\Rightarrow \widetilde{\E}\left[\left(\begin{pmatrix}
        X \\ Y
    \end{pmatrix} - \begin{pmatrix}
        -\Delta^2/2 \\ -u\Delta^2
    \end{pmatrix}\right)\right] }\ = \Delta^2\begin{pmatrix}
  1 & u \\ u & 2u
\end{pmatrix}\begin{pmatrix}
    \alpha\widetilde{\E}[ v(X)] \\
    1 - \alpha \widetilde{\E}[w(Y)]
\end{pmatrix}\,,
\end{align*}
where the last equality follows from derivative calculations in Appendix~\ref{subsec:derivs}. Thus 
\[\widetilde{\E}[Y] = \Delta^2u (1 + \alpha \widetilde{\E}[v] - 2\alpha \widetilde{\E}[w]) = \Delta^2 u(1 - \alpha \widetilde{\E}[w]) \leq \frac{\Delta^2}{2}\,,
\]
where the last equality is from $\widetilde{\E}[v]=\widetilde{\E}[w]$ and the inequality is from $\widetilde{\E}[w] \geq 0, u \leq 1/2$. This implies
\[\widetilde{\E}[w] \leq \frac{p}{1-p}\exp\left(\frac{\Delta^2}{2} + \frac{\Delta^2}{2-\alpha \Delta^2/2}\right)\,,
\]
and thus for \eqref{eq:qlb}, if
\[\frac{p}{1-p}\exp\left(\frac{\Delta^2}{2} + \frac{3\Delta^2}{2 - \alpha \Delta^2/2}\right)\left(3 + ((3\sqrt{2} + 2)\alpha + 2\sqrt{2}-2)\frac{\Delta^2}{1-\alpha \Delta^2/4}\right) \leq 1\,,
\]
we have $\widetilde\E[q(v, w)] > 0$, as desired. The form displayed in Proposition~\ref{prop:sufficient} is obtained by plugging in $\Delta = 2/\sigma$.

\end{enumerate}

\subsubsection{Proof of Lemma~\ref{lem:polynomialbounds}}\label{subsubsec:polynomialbounds}
We use notation such as $\widetilde{\Var}, \widetilde{\Cov}$, etc. to denote that the underlying measure is the tilted measure $\widetilde{\Prob}$ \eqref{eq:tiltedmeasure}. The Hessian lower bound \eqref{eq:hessianlb} for the negative log-density implies the following:\footnote{While the following inequalities are usually stated for Hessian lower bounds that are of the form $cI$ (rather than for anisotropic matrices as in our case), they readily extend to general positive definite matrices $K$ by considering a change of basis $\begin{pmatrix} x \\ y \end{pmatrix} \mapsto K^{1/2}\begin{pmatrix} x \\ y \end{pmatrix}$ and the chain rule.}
\begin{itemize}
    \item By Brascamp-Lieb inequality \citep[Theorem 4.1]{brascamp1976logconcave},
\begin{equation}\label{eq:brascamplieb}
    \widetilde{\Var}[G(X, Y)] \leq \frac{1}{1-\alpha \Delta^2/4}\widetilde{\E}\left[\nabla G(X, Y)^T \left(\Delta^2\begin{pmatrix}
  1 & u \\ u & 2u
\end{pmatrix}\right) \nabla G(X, Y) \right]\,.
\end{equation}
    \item By log-Sobolev inequality \citep[Corollary 5.7.2]{bakry2014markov},
    \begin{equation}\label{eq:logsobolev}
        \widetilde{\mathrm{Ent}}[G(X, Y)^2] \leq \frac{2}{1-\alpha \Delta^2/4}\widetilde{\E}\left[\nabla G(X, Y)^T \left(\Delta^2\begin{pmatrix}
  1 & u \\ u & 2u
\end{pmatrix}\right) \nabla G(X, Y) \right]\,,
    \end{equation}
    where $\widetilde{\mathrm{Ent}}[G(X, Y)^2] \triangleq \widetilde{\E}[G(X, Y)^2 \log G(X, Y)^2] - \widetilde{\E}[G(X, Y)^2] \log \widetilde{\E}[G(X, Y)^2]$.
\end{itemize}

For \eqref{eq:polynomial1}, it is a reverse H\"older inequality which can be obtained from the following standard techniques (see \citep[Equation 15]{fleury2010conc}, for example): \eqref{eq:logsobolev} with $G(x, y) = w(y)^{r/2}$ implies
\[\widetilde{\mathrm{Ent}}[w^r] \leq \frac{2}{1-\alpha \Delta^2/4} \times \widetilde{\E}\left[ 2\Delta^2u \times \frac{r^2}{4}w^r(1-w)^2 \right] \leq \frac{r^2\Delta^2}{2 - \alpha \Delta^2 / 2}\widetilde{\E}[w^r]\,,
\]
since $\nabla G(x, y) = \begin{pmatrix}
    0 \\ (r/2)w^{r/2}(1-w)
\end{pmatrix}$ and $u \leq 1/2, (1-w)^2 \leq 1$. From $\frac{\diff}{\diff r} \left(\frac{1}{r}\log \widetilde{\E}[w^r] \right) = \frac{\widetilde{\mathrm{Ent}}[w^r]}{r^2 \widetilde{\E}[w^r]}$,
integration over $r \in [1, 2]$ yields
\[\frac{1}{2}\log \widetilde{\E}[w^2] - \log \widetilde{\E}[w] \leq \frac{\Delta^2}{2-\alpha \Delta^2/2} \Rightarrow \widetilde{\E}[w^2] \leq (\widetilde{\E}[w])^2\exp\left(\frac{\Delta^2}{1-\alpha \Delta^2/4}\right)\,.
\]
Note that the same arguments yield $\widetilde{\E}[v^2] \leq (\widetilde{\E}[v])^2\exp\left(\frac{\Delta^2}{1-\alpha \Delta^2/4}\right)$.

For \eqref{eq:polynomial2}, we have $\nabla (v-w) = \begin{pmatrix}
    v(1-v) \\
    -w(1-w)
\end{pmatrix}$ and $\widetilde{\E}[v-w]=0$. By \eqref{eq:brascamplieb} with $G(x, y) = v(x)-w(y)$,
\begin{align*}
    \widetilde{\E}[(v-w)^2] = \widetilde{\Var}[v-w] &\leq \frac{\Delta^2}{1-\alpha \Delta^2/4}\widetilde{\E}[v^2(1-v)^2 -2uv(1-v)w(1-w) +2uw^2(1-w)^2] \\
    &\leq \frac{\Delta^2}{1-\alpha \Delta^2/4}\widetilde{\E}[v^2 + w^2] \leq (\widetilde{\E}[v])^2\frac{2\Delta^2}{1-\alpha \Delta^2/4}\exp\left(\frac{\Delta^2}{1-\alpha \Delta^2/4}\right)\,.
\end{align*}
Finally, for \eqref{eq:polynomial3} we have 
\begin{align*}
    |\widetilde{\E}[w(3\alpha + 2-4\alpha w)(v-w)]| 
    &= |\widetilde{\Cov}[w(3\alpha + 2 - 4\alpha w), v-w]|\\
    &\leq \sqrt{\widetilde{\Var}[w(3\alpha + 2 - 4\alpha w)]\widetilde{\Var}[v-w]}\,,
\end{align*}
where the equality is from $\widetilde{\E}[v-w] = 0$ and the inequality is from Cauchy--Schwarz inequality. Then,
\begin{align*}
    \widetilde{\Var}[w(3\alpha+2 - 4\alpha w)] &\leq \frac{1}{1-\alpha \Delta^2/4} \widetilde{\E}[2\Delta^2u \times w^2(1-w)^2(3\alpha + 2 - 8\alpha w)^2] \\
    &\leq \frac{(3\alpha + 2)^2\Delta^2}{1-\alpha \Delta^2/4}  \widetilde{\E}[w^2] \leq (\widetilde{\E}[v])^2\frac{(3\alpha + 2)^2\Delta^2}{1-\alpha \Delta^2/4}\exp\left(\frac{\Delta^2}{1-\alpha \Delta^2/4}\right)\,.
\end{align*}
Here, the first inequality holds from \eqref{eq:brascamplieb}, the second inequality holds from $(1-w)^2(3\alpha + 2 - 8\alpha w)^2 \leq (3\alpha + 2)^2$, and the last inequality is from the earlier bound on $\widetilde{\E}[w^2]$. Combined with the earlier bound \eqref{eq:polynomial2} on $\widetilde{\Var}[v-w]$, we have
\[|\widetilde{\E}[w(3\alpha + 2-4\alpha w)(v-w)]| \leq (\widetilde{\E}[v])^2\frac{\sqrt{2}(3\alpha + 2)\Delta^2}{1-\alpha \Delta^2/4}\exp\left(\frac{\Delta^2}{1-\alpha \Delta^2/4}\right)\,.
\]
This completes the proof.

\subsection{Proof of Proposition~\ref{prop:unimodal}}\label{subsec:unimodal}
Fix $T, p, \sigma, \delta$ and write $\eps$ in \eqref{eq:eps} as $\eps(\alpha)$.
By \citep[Corollary 2]{renyi}, for each $u \in [0, 1/2]$ the map $\alpha \mapsto (\alpha-1)D_\alpha (\Delta(1, 0) \times \Ber(p) + \calN(0, I_2) \| \Delta(u, \sqrt{2u - u^2}) \times \Ber(p) + \calN(0, I_2))$ is convex over $\alpha \in (1, \infty)$. This implies that $\kappa(\alpha) \triangleq (\alpha-1)\rho(\alpha, p, \Delta)$, as a pointwise supremum over $u \in [0, 1/2]$ as in \eqref{eq:1dopt}, is also convex. For
\begin{equation}\label{eq:convexity}
    \eps(\alpha) = \frac{T \kappa(\alpha) + (\alpha-1)\log(\alpha-1) - \alpha\log \alpha - \log \delta}{\alpha-1}\,,
\end{equation}
we have
\[\frac{\partial^2}{\partial \alpha^2}((\alpha-1)\log(\alpha-1) - \alpha \log \alpha - \log \delta) = \frac{1}{\alpha(\alpha-1)} > 0\,,
\]
implying that the numerator of \eqref{eq:convexity} is strictly convex. Hence, for distinct $\alpha, \alpha' > 1$ and $0 < s < 1$,
\begin{align*}
    (s\alpha + (1-s)\alpha' - 1)\eps(s\alpha + (1-s)\alpha') 
    &< s(\alpha-1)\eps(\alpha) + (1-s)(\alpha'-1)\eps(\alpha') \\
    &\leq (s\alpha + (1-s)\alpha'-1)(\eps(\alpha) \vee \eps(\alpha'))\,,
\end{align*}
implying that $\eps$ is strictly quasiconvex and hence has at most one minimizer in $(1, \infty)$.

Finally, we show that there indeed exists a finite minimizer. As $\alpha \mapsto T\rho(\alpha, p, \Delta)$ is increasing by the monotonicity of R\'enyi divergence \citep[Theorem 3]{renyi} and $\alpha \mapsto \log(1-1/\alpha) - \log(\delta \alpha)/(\alpha-1)$ is increasing over $\alpha \in (1/\delta, \infty)$ (from derivative calculation), their sum $\eps(\alpha)$ is also increasing over $\alpha \in (1/\delta, \infty)$. Furthermore, it is straightforward that $\lim_{\alpha \to 1}\eps(\alpha) = \infty$. Combining those, $\eps$ has a unique finite minimizer, as desired.

\subsection{Proof of Proposition~\ref{prop:pexpansion}}\label{subsec:pexpansion}
    Recall that for \eqref{eq:1dopt} we have $\exp((\alpha-1)\rho(\alpha, p, \Delta)) = \max_{u \in [0, 1/2]}F(u)$, where $F$ is as defined in \eqref{eq:ffunction}. In particular, $F$ can be written as
    \[F(u) = \E[(1+pA)^\alpha (1+pB)^{1-\alpha}]\,,
    \]
    where $A \triangleq e^{\langle a, Z \rangle - \|a\|^2/2}-1, B \triangleq e^{\langle b, Z \rangle - \|b\|^2/2}-1$ (with $a = \Delta(1, 0), b = \Delta(u, \sqrt{2u - u^2}), Z \sim \calN(0, I_2)$).

    A key simplification in the small-$p$ regime is that we can specify the maximizer of $F$.

    \begin{claim} For all sufficiently small $p$ (as a function of $\alpha, \Delta$), $u = 1/2$ attains the maximum of $F$.
    \end{claim}
    \begin{proof}
        We consider the Taylor expansion of $F$ in $p$. The first derivative is 
    \[\frac{\partial}{\partial p}F(u) = \E[\alpha A (1+pA)^{\alpha-1}(1+pB)^{1-\alpha} - (\alpha-1)B(1+pA)^{\alpha - 1}(1+pB)^{-\alpha}]\,,
    \]
    where $\frac{\partial}{\partial p}F(u)\Big|_{p = 0} = \alpha \E[A]-(\alpha-1)\E[B]=0$. Similarly, the second derivative is
    \begin{align*}
        \frac{\partial^2}{\partial p^2} F(u)  &= \E[\alpha(\alpha-1)A^2(1+pA)^{\alpha-2}(1+pB)^{1-\alpha} \\
        &\quad- 2\alpha(\alpha-1)AB(1+pA)^{\alpha - 1}(1+pB)^{-\alpha} + \alpha(\alpha-1)B^2(1+pA)^{\alpha-1}(1+pB)^{-\alpha - 1}]\,,
    \end{align*}
    where $\frac{\partial^2}{\partial p^2}F(u)\Big|_{p = 0} = \alpha(\alpha-1)\E[(A-B)^2] = \alpha(\alpha-1)(e^{\Delta^2}-2e^{u \Delta^2}+ e^{2u\Delta^2})$ following from calculations for moment generating functions of Gaussians:
    \begin{equation}\label{eq:gaussianmgf}
        \E[(A-B)^2] = \E[e^{2\langle a, Z \rangle - \|a\|^2} - 2e^{\langle a+b, Z \rangle - \|a\|^2/2 - \|b\|^2/2} + e^{2\langle b, Z \rangle - \|b\|^2}] = e^{\Delta^2}-2e^{u \Delta^2}+ e^{2u\Delta^2}\,.
    \end{equation}
    Then by Taylor's theorem at $p = 0$,
    \[F(u) = 1 + \frac{\alpha(\alpha-1)}{2}(e^{\Delta^2}-2e^{u \Delta^2}+ e^{2u\Delta^2})p^2 + R_{\alpha, \Delta}(p, u)\,,
    \]
    where 
    \[R_{\alpha, \Delta}(p, u) \triangleq \frac{1}{2}\int_{0}^p (z-p)^2 \left(\frac{\partial^3}{\partial p^3 }F(u) \Big|_{p = z}\right) \diff z\,.
    \]
    In particular, let $p_0(\alpha, \Delta) > 0$ be such that $\sup_{p \in [0, p_0(\alpha, \Delta)], u \in (0, 1/2]}|(\partial^4/(\partial u \partial p^3))F(u)| \leq C_{\alpha, \Delta}$ for some $C_{\alpha, \Delta} > 0$. Such $p_0(\alpha, \Delta)$ and $C_{\alpha, \Delta}$ always exist: by Gaussian integration by parts as in \eqref{eq:firstderiv-pre}, we have
    \[\frac{\partial^4}{\partial u \partial p^3}F(u) = \Delta^2\E\left[\left((1+B)^2\frac{\partial^2}{\partial B^2} + (1+A)(1+B)\frac{\partial^2}{\partial A \partial B}\right)\frac{\partial^3}{\partial p^3}\left((1+pA)^{\alpha}(1+pB)^{1-\alpha}\right)\right]\,,
    \]
    whose absolute value is at most a sum of finitely many terms of the form $C_{\alpha, \Delta}'(1-p_0(\alpha, \Delta))^{c}(1 + \E[(1+A)^a(1+B)^b])$ for finite $a, b, c$ and $C'_{\alpha, \Delta}$.
    
    By the mean value theorem,
    \begin{align*} 
        |R_{\alpha, \Delta}(p, 1/2) - R_{\alpha, \Delta}(p, u)| &= \left|\frac{1}{2}\int_{0}^p (z-p)^2 \left( \frac{\partial^3}{\partial p^3 }F(1/2) \Big|_{p = z} - \frac{\partial^3}{\partial p^3 }F(u) \Big|_{p = z}\right) \diff z\right| \\
        &\leq \frac{1}{2}\int_{0}^p (z-p)^2 \left(\sup_{p \in [0, p_0(\alpha, \Delta)], u \in (0, 1/2]}\left| \frac{\partial^4}{\partial u\partial p^3 }F(u) \right| \right)\left(\frac{1}{2}-u\right) \diff z \\
        &\leq \frac{C_{\alpha, \Delta}p^3}{6}\left(\frac{1}{2}-u\right)\,,
    \end{align*}
    which implies
    \begin{align*}
        F(1/2)-F(u) &= \frac{\alpha(\alpha-1)}{2}(e^{\Delta^2} + 2e^{u\Delta^2}-e^{2u\Delta^2}-2e^{\Delta^2/2})p^2 + R_{\alpha, \Delta}(p, 1/2) - R_{\alpha, \Delta}(p, u) \\
        &= \frac{\alpha(\alpha-1)}{2}(e^{\Delta^2/2}-e^{u\Delta^2})(e^{\Delta^2/2}+e^{u\Delta^2}-2)p^2 + R_{\alpha, \Delta}(p, 1/2) - R_{\alpha, \Delta}(p, u) \\
        &\geq \frac{\alpha(\alpha-1)}{2}\Delta^2\left(\frac{1}{2}-u\right)(e^{\Delta^2/2}-1)p^2 - \frac{C_{\alpha, \Delta}p^3}{6}\left(\frac{1}{2}-u\right) \\
        &= \left(\frac{1}{2}-u\right)p^2\left(\frac{\alpha(\alpha-1)}{2}\Delta^2(e^{\Delta^2/2}-1) - \frac{C_{\alpha, \Delta}p}{6}\right)\,,
    \end{align*}
    implying that for all sufficiently small $p$ (as a function of $\alpha, \Delta$), $F$ is maximized at $u = 1/2$ as desired.
    \end{proof}
     Given the claim, we consider the third-order Taylor expansion. From
    \begin{align*}
        \frac{\partial^3}{\partial p^3}F(u) &= \alpha(\alpha-1)\E\Big[ (\alpha-2)A^3(1+pA)^{\alpha-3}(1+pB)^{1-\alpha}  - 3(\alpha-1)A^2B(1+pA)^{\alpha-2}(1+pB)^{-\alpha} \\
        &\quad+ 3\alpha AB^2 (1+pA)^{\alpha-1}(1+pB)^{-\alpha - 1} - (\alpha+1)B^3(1+pA)^{\alpha}(1+pB)^{-\alpha - 2}\Big]\,, \\
        \frac{\partial^4}{\partial p^4}F(u) &= \alpha(\alpha-1)\E \Big[(\alpha-2)(\alpha-3) A^4(1+pA)^{\alpha-4}(1+pB)^{1-\alpha} \\
        &\quad- 4(\alpha-1)(\alpha-2) A^3B(1+pA)^{\alpha-3}(1+pB)^{-\alpha} + 6\alpha(\alpha-1)A^2B^2(1+pA)^{\alpha-2}(1+pB)^{-\alpha-1}\\
        &\quad - 4(\alpha+1)\alpha AB^3(1+pA)^{\alpha-1}(1+pB)^{-\alpha-2}+ (\alpha+2)(\alpha+1) B^4(1+pA)^{\alpha}(1+pB)^{-\alpha-3}\Big]\,,
    \end{align*}
    it readily follows that $\sup_{p \in [0, p_0(\alpha, \Delta)], u \in [0, 1/2]}\Big|\frac{\partial^4}{\partial p^4}F(u)\Big| = O_{\alpha, \Delta}(1)$ uniformly over $u \in [0,1 /2]$. Furthermore,
    \[\frac{\partial^3}{\partial p^3}F(u) \Bigg|_{p=0} = \alpha(\alpha-1)[(\alpha-2)\E[A^3]-3(\alpha-1)\E[A^2B]+ 3\alpha \E[AB^2]-(\alpha+1)\E[B^3]]\,,
    \]
    where
    \begin{align*}
        \E[A^3] &= e^{3\Delta^2}-3e^{\Delta^2}+2, \quad \E[A^2B] = e^{(2u+1)\Delta^2}-2e^{u\Delta^2}-e^{\Delta^2} + 2, \\
        \E[AB^2] &= e^{4u\Delta^2}-e^{2u\Delta^2}-2e^{u\Delta^2}+2, \quad \E[B^3] = e^{6u\Delta^2}-3e^{2u\Delta^2} + 2\,,
    \end{align*}
    from calculations similar to \eqref{eq:gaussianmgf} based on moment generating functions. Then by Taylor's theorem,
    \begin{align*}
        F(u) &= 1 + \frac{\alpha(\alpha-1)}{2}\left(e^{\Delta^2} - 2e^{u\Delta^2} + e^{2u\Delta^2}\right)p^2 \\
        &\quad + \frac{\alpha(\alpha-1)}{6}\left((\alpha-2)e^{3\Delta^2} - 3(\alpha-1)e^{(2u+1)\Delta^2} + 3\alpha e^{4u\Delta^2} - (\alpha+1)e^{6u\Delta^2} + 3e^{\Delta^2} + 3e^{2u\Delta^2}-6e^{u\Delta^2}\right)p^3 \\
        &\quad + O_{\alpha, \Delta}(p^4)\,.
    \end{align*}
    By the earlier claim, it suffices to consider $u = 1/2$ for maximizing $F$. Then
    \[\exp((\alpha-1)\rho(\alpha, p, \Delta)) = 1 + \alpha(\alpha-1)(e^{\Delta^2}-e^{\Delta^2/2})p^2 - \frac{\alpha(\alpha-1)}{2}(e^{3\Delta^2}-e^{2\Delta^2}-2e^{\Delta^2}+2e^{\Delta^2/2}) + O_{\alpha, \Delta}(p^4)\,,
    \]
    from $\exp((\alpha-1)\rho(\alpha, p, \Delta)) = \max_{u \in [0, 1/2]}F(u)$.  Plugging in $\Delta = 2/\sigma$, we obtain the first part of Proposition~\ref{prop:pexpansion}.

    The second part is based on the analytical expressions in \citep{birrell2024rdp}; here we only record relevant details. From \citep[Theorem 3.4]{birrell2024rdp} with order $m \geq 4$, the coefficient of $p^3$ for the asymptotic expansion of  $\exp((\alpha-1)\rho_\mathrm{B})$ is given as
    \begin{equation}\label{eq:birrell-pexpansion1}
    \begin{aligned}
        &\frac{\alpha^2(\alpha - 1)\sqrt{M_{\sigma, 2}M_{\sigma, 4}}}{2} \\
        &\quad+ \frac{\alpha^2(\alpha-1)\sqrt{M_{\sigma, 2}M_{\sigma, 4}}}{6} \times  \sum_{j=0}^3 \binom{3}{j} \left|\frac{\alpha}{\alpha - 1}\left(\prod_{\ell=0}^{j-1} \left(1-\frac{\ell}{\alpha}\right)\right)\left(\prod_{\ell=0}^{2-j} \left(1+\frac{\ell-1}{\alpha}\right)\right) - 1\right|\,,
    \end{aligned}    
    \end{equation}
    
    where $M_{\sigma, 2} = e^{4/\sigma^2} - 1, M_{\sigma, 4} = e^{24/\sigma^2}-4e^{12/\sigma^2}+6e^{4/\sigma^2}-3$
    from \citep[Equation 41]{birrell2024rdp}. Thus \eqref{eq:birrell-pexpansion1} is equal to
    \[\frac{(\alpha+2)\alpha(\alpha-1)}{2}\sqrt{(e^{4/\sigma^2}-1)(e^{24/\sigma^2}-4e^{12/\sigma^2}+6e^{4/\sigma^2}-3)}\,,
    \]
    as desired. Converting this expansion into an expansion for $\rho_{\mathrm{B}}$, we obtain the second part.

\subsection{Proof of Corollary~\ref{cor:lim}}\label{subsec:lim}
We first show $\mu^2 > \lambda$. From their definitions (see Lemma~\ref{lem:clt} and Proposition~\ref{prop:lim}), this is equivalent to showing
\[e^{4/\sigma^2}\Phi(3/\sigma) + 3\Phi(-1/\sigma) -2 > e^{4/\sigma^2} - e^{2/\sigma^2}\,.
\]
Let $f(y) \triangleq e^{4y^2}(\Phi(3y)-1) + 3\Phi(-y) + e^{2y^2}-2$. Then
\begin{align*}
    f'(y) &= 8ye^{4y^2}(\Phi(3y)-1) + 3e^{4y^2}\phi(3y) - 3\phi(-y) + 4ye^{2y^2} \\
    &= 8ye^{4y^2} (\Phi(3y)-1) + 4ye^{2y^2} \\
    &= 4ye^{2y^2}(1 - 2e^{2y^2}\Phi(-3y)) \geq 4ye^{2y^2}(1-e^{-5y^2/2}) > 0\,,
\end{align*}
for all $y > 0$; here the second equality holds from $3e^{4y^2}\phi(3y) = 3\phi(y) = 3\phi(-y)$, and the penultimate inequality holds from standard Gaussian tail bound $\Phi(-3y) \leq e^{-(9/2)y^2}/2$. Since $f(0) = 0$, this implies that $f(1/\sigma) > 0$ for all $\sigma > 0$, proving $\mu^2 > \lambda$.

For the bound on $\log \delta(\eps) - \log \delta_{\mathrm{L}}(\eps)$, we combine an upper bound on $\log \delta(\eps)$ and a lower bound on $\log\delta_{\mathrm{L}}(\eps)$:
\begin{align}
    \log \delta(\eps) &\leq -\frac{\eps^2}{2\lambda} + \frac{\eps}{2} -\log \eps + O_{\lambda}(1)\,, \label{eq:deltaub} \\
    \log \delta_{\mathrm{L}}(\eps) &\geq -\frac{\eps^2}{2\mu^2} + \frac{\eps}{2} -2 \log \eps - O_{\mu}(1)\,. \label{eq:deltahlb}
\end{align}
Corollary~\ref{cor:lim} clearly follows from these two inequalities. From Proposition~\ref{prop:lim}, we have
\[\delta(\eps) \leq \frac{\lambda}{\eps} \exp\left(-\frac{\eps^2}{2\lambda} + \frac{\eps}{2}\right)\,,
\]
which implies \eqref{eq:deltaub}. For \eqref{eq:deltahlb}, recall the following standard inequality:
\[\left(\frac{1}{x} - \frac{1}{x^3}\right)\phi(x)\leq \Phi(-x) \leq \frac{1}{x}\phi(x) \text{ for all } x> 0\,.
\]
For $\eps \geq \mu^2 \vee 16$, using this inequality for Lemma~\ref{lem:clt} yields
\begin{align*}
    \delta_{\mathrm{L}}(\eps) &\geq \left(\frac{1}{\eps / \mu - \mu / 2} - \frac{1}{(\eps / \mu - \mu / 2)^3}\right)\phi\left(\frac{\eps}{\mu} - \frac{\mu}{2}\right) - e^{\eps} \times \frac{1}{\eps/\mu + \mu / 2}\phi\left(\frac{\eps}{\mu} + \frac{\mu}{2}\right) \\
    &= \left(\frac{1}{\eps / \mu - \mu / 2} - \frac{1}{\eps/\mu + \mu /2} - \frac{1}{(\eps / \mu - \mu / 2)^3}\right) \phi\left(\frac{\eps}{\mu} - \frac{\mu}{2}\right) \\
    &\geq \left(\frac{\mu}{\eps^2/\mu^2 - \mu^2/4} - \frac{1}{(\eps / (2\mu))^3}\right) \phi\left(\frac{\eps}{\mu} - \frac{\mu}{2}\right)   \\
    &\geq \left(\frac{\mu^3}{\eps^2} - \frac{\mu^3}{2\eps^2}\right) \phi\left(\frac{\eps}{\mu} - \frac{\mu}{2}\right) = \frac{\mu^3}{2\eps^2}\phi\left(\frac{\eps}{\mu} - \frac{\mu}{2}\right) \,,
\end{align*}
where the second inequality holds from $\eps/\mu - \mu/2 \geq \eps / (2\mu) \Leftrightarrow \eps \geq \mu^2$, and the last inequality holds from $\mu^3/(2\eps^2) \geq 8\mu^3/\eps^3 \Leftrightarrow \eps \geq 16$. This implies \eqref{eq:deltahlb}, as desired.

\section{Deferred details for numerical comparisons}\label{sec:numericsdetails}

\subsection{Additional results}\label{subsec:numerics-full}

We record the values of $\eps$ for $(\eps, \delta)$-DP guarantees from different accountants over a grid of hyperparameters $T \in \{250, 500, 1000, 2000, 4000, 8000, 16000, 32000, 64000\}, p \in \{0.0025, 0.01, 0.04\}, \sigma \in \{1, 2, 3, 4, 6, 8\}$ and $\delta \in \{10^{-6}, 10^{-8}\}$. We omit values of $\eps$ for sets of hyperparameters where all accounting methods yield $\eps \geq 20$, as such large values are often considered as being outside of practically meaningful regime. In some cases, the bounds from \cite{birrell2024rdp} are excessively large due to the remainder terms in their Taylor expansion; we only record that $\eps \geq 100$ in those cases. As the main interest of this paper is to compare R\'enyi accountants against lossless accountants, the column named ``Relative difference'' in the tables records how small $\eps$ from our R\'enyi accountant (which always improves over other R\'enyi accountants covered here) is relative to $\eps$ from lossless accountants, in percentage. In particular, the percentages are colored \textcolor{Green}{green} if they are negative (i.e., $\eps$ from our R\'enyi accountant is smaller) and colored \textcolor{Red}{red} otherwise (i.e., $\eps$ from our R\'enyi accountant is larger). Our code for reproducing the numerics can be found at: \url{https://github.com/jinhobok/renyi_vs_lossless_dpsgd}.

The results are presented in Tables~\ref{table:first} through \ref{table:last}. Across those hyperparameter choices, our privacy bounds improve over those from lossless accountants, except for particular cases (if $\sigma = 1$, or $T$ is small). In particular, those improvements generally become more prominent as the number of iterations increases, allowing longer training with better privacy guarantees. Relatedly, many of the settings where lossless accountants provide better privacy guarantees are when $T$ is small, where the accounted privacy $\eps$ is already small (e.g., less than $1$) across different accountants. Altogether, those results illustrate that a R\'enyi accountant can provide better privacy guarantees than lossless accountants.

\raggedbottom

\begin{table}[H]
\centering
\caption{Accounted privacy $\varepsilon$ for $\delta=10^{-6}$ and $p=0.0025$.}
\label{table:first}
\footnotesize
\begin{tabular}{|c|c|c|c|c|c|c|}
\hline
\multirow{2}{*}{$\sigma$} & \multirow{2}{*}{$T$} & \multicolumn{3}{c|}{R\'enyi accountants} & \multirow{2}{*}{Lossless accountants} & \multirow{2}{*}{Relative difference} \\
\cline{3-5}
 &  & Our bounds & Birrell et al. & Wang et al. &  &  \\
\hline
\multirow{8}{*}{$1$} & $250$ & $6.40$ & $\geq 100$ & $6.84$ & $5.10$ & $\textcolor{Red}{+25.48\%}$ \\
 & $500$ & $6.84$ & $\geq 100$ & $7.73$ & $5.68$ & $\textcolor{Red}{+20.32\%}$ \\
 & $1000$ & $7.72$ & $\geq 100$ & $9.51$ & $6.53$ & $\textcolor{Red}{+18.26\%}$ \\
 & $2000$ & $9.49$ & $\geq 100$ & $13.06$ & $7.80$ & $\textcolor{Red}{+21.63\%}$ \\
 & $4000$ & $13.03$ & $\geq 100$ & $15.16$ & $9.76$ & $\textcolor{Red}{+33.46\%}$ \\
 & $8000$ & $15.77$ & $\geq 100$ & $17.89$ & $12.86$ & $\textcolor{Red}{+22.59\%}$ \\
 & $16000$ & $19.11$ & $\geq 100$ & $23.35$ & $17.87$ & $\textcolor{Red}{+6.96\%}$ \\
\cline{3-7}
 & $\geq 32000$ & \multicolumn{5}{c|}{$\geq 20$ for all accountants} \\
\hline
\multirow{9}{*}{$2$} & $250$ & $1.03$ & $1.43$ & $1.10$ & $0.33$ & $\textcolor{Red}{+215.71\%}$ \\
 & $500$ & $1.08$ & $1.52$ & $1.15$ & $0.42$ & $\textcolor{Red}{+155.55\%}$ \\
 & $1000$ & $1.12$ & $1.66$ & $1.25$ & $0.58$ & $\textcolor{Red}{+93.95\%}$ \\
 & $2000$ & $1.19$ & $1.78$ & $1.45$ & $0.81$ & $\textcolor{Red}{+46.97\%}$ \\
 & $4000$ & $1.34$ & $2.02$ & $1.86$ & $1.16$ & $\textcolor{Red}{+15.32\%}$ \\
 & $8000$ & $1.64$ & $2.36$ & $2.61$ & $1.68$ & $\textcolor{Green}{-2.74\%}$ \\
 & $16000$ & $2.23$ & $2.90$ & $3.79$ & $2.45$ & $\textcolor{Green}{-9.16\%}$ \\
 & $32000$ & $3.24$ & $3.92$ & $5.57$ & $3.59$ & $\textcolor{Green}{-9.80\%}$ \\
 & $64000$ & $4.77$ & $5.50$ & $8.26$ & $5.31$ & $\textcolor{Green}{-10.34\%}$ \\
\hline
\multirow{9}{*}{$3$} & $250$ & $0.40$ & $0.46$ & $0.43$ & $0.14$ & $\textcolor{Red}{+187.32\%}$ \\
 & $500$ & $0.41$ & $0.48$ & $0.48$ & $0.20$ & $\textcolor{Red}{+107.41\%}$ \\
 & $1000$ & $0.43$ & $0.52$ & $0.58$ & $0.28$ & $\textcolor{Red}{+53.13\%}$ \\
 & $2000$ & $0.48$ & $0.57$ & $0.77$ & $0.41$ & $\textcolor{Red}{+18.47\%}$ \\
 & $4000$ & $0.58$ & $0.68$ & $1.10$ & $0.59$ & $\textcolor{Green}{-1.01\%}$ \\
 & $8000$ & $0.79$ & $0.88$ & $1.59$ & $0.86$ & $\textcolor{Green}{-8.12\%}$ \\
 & $16000$ & $1.14$ & $1.23$ & $2.31$ & $1.24$ & $\textcolor{Green}{-8.45\%}$ \\
 & $32000$ & $1.65$ & $1.76$ & $3.36$ & $1.81$ & $\textcolor{Green}{-8.72\%}$ \\
 & $64000$ & $2.41$ & $2.53$ & $4.94$ & $2.65$ & $\textcolor{Green}{-9.05\%}$ \\
\hline
\multirow{9}{*}{$4$} & $250$ & $0.21$ & $0.23$ & $0.24$ & $0.09$ & $\textcolor{Red}{+128.03\%}$ \\
 & $500$ & $0.22$ & $0.24$ & $0.28$ & $0.13$ & $\textcolor{Red}{+66.77\%}$ \\
 & $1000$ & $0.24$ & $0.27$ & $0.37$ & $0.19$ & $\textcolor{Red}{+27.26\%}$ \\
 & $2000$ & $0.28$ & $0.32$ & $0.53$ & $0.27$ & $\textcolor{Red}{+4.31\%}$ \\
 & $4000$ & $0.37$ & $0.41$ & $0.76$ & $0.39$ & $\textcolor{Green}{-5.40\%}$ \\
 & $8000$ & $0.53$ & $0.58$ & $1.10$ & $0.57$ & $\textcolor{Green}{-6.18\%}$ \\
 & $16000$ & $0.77$ & $0.82$ & $1.59$ & $0.83$ & $\textcolor{Green}{-6.40\%}$ \\
 & $32000$ & $1.12$ & $1.17$ & $2.31$ & $1.20$ & $\textcolor{Green}{-6.67\%}$ \\
 & $64000$ & $1.63$ & $1.68$ & $3.38$ & $1.75$ & $\textcolor{Green}{-6.92\%}$ \\
\hline
\multirow{9}{*}{$6$} & $250$ & $0.09$ & $0.09$ & $0.12$ & $0.05$ & $\textcolor{Red}{+63.54\%}$ \\
 & $500$ & $0.10$ & $0.10$ & $0.16$ & $0.08$ & $\textcolor{Red}{+26.27\%}$ \\
 & $1000$ & $0.12$ & $0.13$ & $0.23$ & $0.11$ & $\textcolor{Red}{+5.53\%}$ \\
 & $2000$ & $0.16$ & $0.17$ & $0.33$ & $0.16$ & $\textcolor{Green}{-2.07\%}$ \\
 & $4000$ & $0.23$ & $0.25$ & $0.47$ & $0.23$ & $\textcolor{Green}{-2.63\%}$ \\
 & $8000$ & $0.33$ & $0.35$ & $0.68$ & $0.34$ & $\textcolor{Green}{-2.84\%}$ \\
 & $16000$ & $0.48$ & $0.50$ & $0.99$ & $0.49$ & $\textcolor{Green}{-3.06\%}$ \\
 & $32000$ & $0.69$ & $0.71$ & $1.43$ & $0.71$ & $\textcolor{Green}{-3.28\%}$ \\
 & $64000$ & $1.00$ & $1.02$ & $2.08$ & $1.03$ & $\textcolor{Green}{-3.53\%}$ \\
\hline
\multirow{9}{*}{$8$} & $250$ & $0.05$ & $0.05$ & $0.08$ & $0.04$ & $\textcolor{Red}{+33.10\%}$ \\
 & $500$ & $0.06$ & $0.07$ & $0.12$ & $0.05$ & $\textcolor{Red}{+10.21\%}$ \\
 & $1000$ & $0.08$ & $0.09$ & $0.17$ & $0.08$ & $\textcolor{Red}{+1.16\%}$ \\
 & $2000$ & $0.11$ & $0.12$ & $0.24$ & $0.11$ & $\textcolor{Red}{+0.30\%}$ \\
 & $4000$ & $0.16$ & $0.18$ & $0.34$ & $0.17$ & $\textcolor{Green}{-0.28\%}$ \\
 & $8000$ & $0.24$ & $0.25$ & $0.50$ & $0.24$ & $\textcolor{Green}{-0.61\%}$ \\
 & $16000$ & $0.34$ & $0.36$ & $0.72$ & $0.35$ & $\textcolor{Green}{-0.83\%}$ \\
 & $32000$ & $0.50$ & $0.51$ & $1.04$ & $0.50$ & $\textcolor{Green}{-1.05\%}$ \\
 & $64000$ & $0.72$ & $0.73$ & $1.50$ & $0.73$ & $\textcolor{Green}{-1.27\%}$ \\
\hline
\end{tabular}
\end{table}

\begin{table}[H]
\centering
\caption{Accounted privacy $\varepsilon$ for $\delta=10^{-6}$ and $p=0.01$.}
\footnotesize
\begin{tabular}{|c|c|c|c|c|c|c|}
\hline
\multirow{2}{*}{$\sigma$} & \multirow{2}{*}{$T$} & \multicolumn{3}{c|}{R\'enyi accountants} & \multirow{2}{*}{Lossless accountants} & \multirow{2}{*}{Relative difference} \\
\cline{3-5}
 &  & Our bounds & Birrell et al. & Wang et al. &  &  \\
\hline
\multirow{4}{*}{$1$} & $250$ & $13.77$ & $\geq 100$ & $15.14$ & $10.27$ & $\textcolor{Red}{+34.03\%}$ \\
 & $500$ & $15.11$ & $\geq 100$ & $17.86$ & $12.90$ & $\textcolor{Red}{+17.13\%}$ \\
 & $1000$ & $17.79$ & $\geq 100$ & $23.29$ & $16.96$ & $\textcolor{Red}{+4.85\%}$ \\
\cline{3-7}
 & $\geq 2000$ & \multicolumn{5}{c|}{$\geq 20$ for all accountants} \\
\hline
\multirow{9}{*}{$2$} & $250$ & $1.77$ & $3.27$ & $2.21$ & $1.34$ & $\textcolor{Red}{+31.35\%}$ \\
 & $500$ & $1.99$ & $3.96$ & $2.89$ & $1.80$ & $\textcolor{Red}{+10.68\%}$ \\
 & $1000$ & $2.44$ & $4.93$ & $4.02$ & $2.51$ & $\textcolor{Green}{-3.03\%}$ \\
 & $2000$ & $3.32$ & $6.00$ & $5.76$ & $3.61$ & $\textcolor{Green}{-7.93\%}$ \\
 & $4000$ & $4.75$ & $8.14$ & $8.42$ & $5.28$ & $\textcolor{Green}{-10.02\%}$ \\
 & $8000$ & $7.00$ & $10.89$ & $12.63$ & $7.84$ & $\textcolor{Green}{-10.77\%}$ \\
 & $16000$ & $10.48$ & $15.83$ & $19.31$ & $11.83$ & $\textcolor{Green}{-11.44\%}$ \\
 & $32000$ & $15.83$ & $23.23$ & $29.82$ & $18.16$ & $\textcolor{Green}{-12.82\%}$ \\
\cline{3-7}
 & $64000$ & \multicolumn{5}{c|}{$\geq 20$ for all accountants} \\
\hline
\multirow{9}{*}{$3$} & $250$ & $0.69$ & $0.96$ & $1.17$ & $0.61$ & $\textcolor{Red}{+13.75\%}$ \\
 & $500$ & $0.85$ & $1.18$ & $1.67$ & $0.86$ & $\textcolor{Green}{-1.45\%}$ \\
 & $1000$ & $1.15$ & $1.55$ & $2.39$ & $1.25$ & $\textcolor{Green}{-7.56\%}$ \\
 & $2000$ & $1.65$ & $2.09$ & $3.45$ & $1.81$ & $\textcolor{Green}{-8.62\%}$ \\
 & $4000$ & $2.41$ & $2.90$ & $5.02$ & $2.64$ & $\textcolor{Green}{-8.93\%}$ \\
 & $8000$ & $3.51$ & $4.11$ & $7.44$ & $3.88$ & $\textcolor{Green}{-9.44\%}$ \\
 & $16000$ & $5.18$ & $5.88$ & $11.17$ & $5.75$ & $\textcolor{Green}{-9.98\%}$ \\
 & $32000$ & $7.75$ & $8.58$ & $16.81$ & $8.61$ & $\textcolor{Green}{-10.00\%}$ \\
 & $64000$ & $11.72$ & $12.83$ & $26.75$ & $13.07$ & $\textcolor{Green}{-10.29\%}$ \\
\hline
\multirow{9}{*}{$4$} & $250$ & $0.41$ & $0.54$ & $0.80$ & $0.40$ & $\textcolor{Red}{+3.35\%}$ \\
 & $500$ & $0.55$ & $0.71$ & $1.14$ & $0.57$ & $\textcolor{Green}{-3.72\%}$ \\
 & $1000$ & $0.78$ & $0.96$ & $1.63$ & $0.83$ & $\textcolor{Green}{-5.36\%}$ \\
 & $2000$ & $1.12$ & $1.32$ & $2.35$ & $1.20$ & $\textcolor{Green}{-6.58\%}$ \\
 & $4000$ & $1.63$ & $1.84$ & $3.42$ & $1.75$ & $\textcolor{Green}{-6.81\%}$ \\
 & $8000$ & $2.37$ & $2.62$ & $5.01$ & $2.55$ & $\textcolor{Green}{-7.16\%}$ \\
 & $16000$ & $3.46$ & $3.76$ & $7.44$ & $3.75$ & $\textcolor{Green}{-7.59\%}$ \\
 & $32000$ & $5.10$ & $5.49$ & $11.20$ & $5.55$ & $\textcolor{Green}{-8.08\%}$ \\
 & $64000$ & $7.63$ & $8.08$ & $16.92$ & $8.31$ & $\textcolor{Green}{-8.16\%}$ \\
\hline
\multirow{9}{*}{$6$} & $250$ & $0.24$ & $0.30$ & $0.49$ & $0.23$ & $\textcolor{Red}{+0.66\%}$ \\
 & $500$ & $0.34$ & $0.41$ & $0.70$ & $0.34$ & $\textcolor{Green}{-0.74\%}$ \\
 & $1000$ & $0.48$ & $0.56$ & $1.01$ & $0.49$ & $\textcolor{Green}{-1.84\%}$ \\
 & $2000$ & $0.69$ & $0.77$ & $1.45$ & $0.71$ & $\textcolor{Green}{-2.73\%}$ \\
 & $4000$ & $0.99$ & $1.09$ & $2.10$ & $1.03$ & $\textcolor{Green}{-3.39\%}$ \\
 & $8000$ & $1.44$ & $1.55$ & $3.06$ & $1.50$ & $\textcolor{Green}{-3.65\%}$ \\
 & $16000$ & $2.10$ & $2.22$ & $4.49$ & $2.19$ & $\textcolor{Green}{-3.94\%}$ \\
 & $32000$ & $3.07$ & $3.21$ & $6.62$ & $3.21$ & $\textcolor{Green}{-4.25\%}$ \\
 & $64000$ & $4.52$ & $4.68$ & $9.90$ & $4.73$ & $\textcolor{Green}{-4.44\%}$ \\
\hline
\multirow{9}{*}{$8$} & $250$ & $0.17$ & $0.21$ & $0.35$ & $0.17$ & $\textcolor{Red}{+2.40\%}$ \\
 & $500$ & $0.24$ & $0.29$ & $0.51$ & $0.24$ & $\textcolor{Red}{+1.28\%}$ \\
 & $1000$ & $0.35$ & $0.39$ & $0.73$ & $0.35$ & $\textcolor{Red}{+0.40\%}$ \\
 & $2000$ & $0.50$ & $0.55$ & $1.05$ & $0.50$ & $\textcolor{Green}{-0.31\%}$ \\
 & $4000$ & $0.72$ & $0.78$ & $1.51$ & $0.73$ & $\textcolor{Green}{-0.90\%}$ \\
 & $8000$ & $1.04$ & $1.10$ & $2.20$ & $1.05$ & $\textcolor{Green}{-1.30\%}$ \\
 & $16000$ & $1.51$ & $1.58$ & $3.21$ & $1.53$ & $\textcolor{Green}{-1.51\%}$ \\
 & $32000$ & $2.20$ & $2.28$ & $4.72$ & $2.24$ & $\textcolor{Green}{-1.76\%}$ \\
 & $64000$ & $3.22$ & $3.31$ & $6.98$ & $3.28$ & $\textcolor{Green}{-1.97\%}$ \\
\hline
\end{tabular}
\end{table}

\begin{table}[H]
\centering
\caption{Accounted privacy $\varepsilon$ for $\delta=10^{-6}$ and $p=0.04$.}
\footnotesize
\begin{tabular}{|c|c|c|c|c|c|c|}
\hline
\multirow{2}{*}{$\sigma$} & \multirow{2}{*}{$T$} & \multicolumn{3}{c|}{R\'enyi accountants} & \multirow{2}{*}{Lossless accountants} & \multirow{2}{*}{Relative difference} \\
\cline{3-5}
 &  & Our bounds & Birrell et al. & Wang et al. &  &  \\
\hline
\multirow{1}{*}{$1$} & $\geq 250$ & \multicolumn{5}{c|}{$\geq 20$ for all accountants} \\
\hline
\multirow{5}{*}{$2$} & $250$ & $5.45$ & $16.66$ & $9.07$ & $5.45$ & $\textcolor{Red}{+0.01\%}$ \\
 & $500$ & $7.31$ & $20.89$ & $13.02$ & $7.79$ & $\textcolor{Green}{-6.22\%}$ \\
 & $1000$ & $10.59$ & $29.34$ & $20.08$ & $11.48$ & $\textcolor{Green}{-7.82\%}$ \\
 & $2000$ & $15.22$ & $46.26$ & $29.75$ & $17.37$ & $\textcolor{Green}{-12.39\%}$ \\
\cline{3-7}
 & $\geq 4000$ & \multicolumn{5}{c|}{$\geq 20$ for all accountants} \\
\hline
\multirow{7}{*}{$3$} & $250$ & $2.58$ & $4.63$ & $5.31$ & $2.66$ & $\textcolor{Green}{-3.07\%}$ \\
 & $500$ & $3.63$ & $6.11$ & $7.70$ & $3.86$ & $\textcolor{Green}{-6.02\%}$ \\
 & $1000$ & $5.22$ & $8.37$ & $11.52$ & $5.68$ & $\textcolor{Green}{-8.06\%}$ \\
 & $2000$ & $7.68$ & $11.49$ & $17.09$ & $8.46$ & $\textcolor{Green}{-9.27\%}$ \\
 & $4000$ & $11.53$ & $17.02$ & $26.73$ & $12.80$ & $\textcolor{Green}{-9.96\%}$ \\
 & $8000$ & $17.42$ & $24.96$ & $41.03$ & $19.73$ & $\textcolor{Green}{-11.69\%}$ \\
\cline{3-7}
 & $\geq 16000$ & \multicolumn{5}{c|}{$\geq 20$ for all accountants} \\
\hline
\multirow{8}{*}{$4$} & $250$ & $1.71$ & $2.57$ & $3.55$ & $1.75$ & $\textcolor{Green}{-2.10\%}$ \\
 & $500$ & $2.43$ & $3.43$ & $5.15$ & $2.54$ & $\textcolor{Green}{-4.22\%}$ \\
 & $1000$ & $3.50$ & $4.71$ & $7.62$ & $3.72$ & $\textcolor{Green}{-5.82\%}$ \\
 & $2000$ & $5.09$ & $6.62$ & $11.40$ & $5.50$ & $\textcolor{Green}{-7.47\%}$ \\
 & $4000$ & $7.57$ & $9.46$ & $17.06$ & $8.21$ & $\textcolor{Green}{-7.84\%}$ \\
 & $8000$ & $11.43$ & $13.77$ & $26.96$ & $12.44$ & $\textcolor{Green}{-8.09\%}$ \\
 & $16000$ & $17.31$ & $21.59$ & $41.49$ & $19.16$ & $\textcolor{Green}{-9.67\%}$ \\
\cline{3-7}
 & $\geq 32000$ & \multicolumn{5}{c|}{$\geq 20$ for all accountants} \\
\hline
\multirow{9}{*}{$6$} & $250$ & $1.03$ & $1.40$ & $2.16$ & $1.03$ & $\textcolor{Red}{+0.25\%}$ \\
 & $500$ & $1.48$ & $1.87$ & $3.13$ & $1.49$ & $\textcolor{Green}{-1.19\%}$ \\
 & $1000$ & $2.13$ & $2.57$ & $4.56$ & $2.18$ & $\textcolor{Green}{-2.29\%}$ \\
 & $2000$ & $3.08$ & $3.61$ & $6.69$ & $3.19$ & $\textcolor{Green}{-3.27\%}$ \\
 & $4000$ & $4.52$ & $5.18$ & $9.98$ & $4.71$ & $\textcolor{Green}{-3.94\%}$ \\
 & $8000$ & $6.66$ & $7.56$ & $15.06$ & $7.01$ & $\textcolor{Green}{-4.90\%}$ \\
 & $16000$ & $9.99$ & $11.26$ & $24.17$ & $10.56$ & $\textcolor{Green}{-5.39\%}$ \\
 & $32000$ & $15.15$ & $16.72$ & $36.49$ & $16.16$ & $\textcolor{Green}{-6.23\%}$ \\
\cline{3-7}
 & $ 64000$ & \multicolumn{5}{c|}{$\geq 20$ for all accountants} \\
\hline
\multirow{9}{*}{$8$} & $250$ & $0.74$ & $0.97$ & $1.55$ & $0.73$ & $\textcolor{Red}{+1.89\%}$ \\
 & $500$ & $1.06$ & $1.30$ & $2.24$ & $1.05$ & $\textcolor{Red}{+0.78\%}$ \\
 & $1000$ & $1.52$ & $1.79$ & $3.25$ & $1.53$ & $\textcolor{Green}{-0.17\%}$ \\
 & $2000$ & $2.21$ & $2.52$ & $4.76$ & $2.23$ & $\textcolor{Green}{-0.93\%}$ \\
 & $4000$ & $3.22$ & $3.59$ & $7.03$ & $3.27$ & $\textcolor{Green}{-1.51\%}$ \\
 & $8000$ & $4.73$ & $5.18$ & $10.54$ & $4.83$ & $\textcolor{Green}{-2.08\%}$ \\
 & $16000$ & $7.00$ & $7.64$ & $15.92$ & $7.19$ & $\textcolor{Green}{-2.68\%}$ \\
 & $32000$ & $10.53$ & $11.42$ & $25.63$ & $10.85$ & $\textcolor{Green}{-2.96\%}$ \\
 & $64000$ & $15.97$ & $17.07$ & $38.84$ & $16.63$ & $\textcolor{Green}{-3.98\%}$ \\
\hline
\end{tabular}
\end{table}

\begin{table}[H]
\centering
\caption{Accounted privacy $\varepsilon$ for $\delta=10^{-8}$ and $p=0.0025$.}
\footnotesize
\begin{tabular}{|c|c|c|c|c|c|c|}
\hline
\multirow{2}{*}{$\sigma$} & \multirow{2}{*}{$T$} & \multicolumn{3}{c|}{R\'enyi accountants} & \multirow{2}{*}{Lossless accountants} & \multirow{2}{*}{Relative difference} \\
\cline{3-5}
 &  & Our bounds & Birrell et al. & Wang et al. &  &  \\
\hline
\multirow{7}{*}{$1$} & $250$ & $8.70$ & $\geq 100$ & $9.14$ & $7.11$ & $\textcolor{Red}{+22.41\%}$ \\
 & $500$ & $9.14$ & $\geq 100$ & $10.03$ & $7.83$ & $\textcolor{Red}{+16.75\%}$ \\
 & $1000$ & $10.03$ & $\geq 100$ & $11.81$ & $8.85$ & $\textcolor{Red}{+13.31\%}$ \\
 & $2000$ & $11.79$ & $\geq 100$ & $15.36$ & $10.32$ & $\textcolor{Red}{+14.31\%}$ \\
 & $4000$ & $15.33$ & $\geq 100$ & $19.76$ & $12.53$ & $\textcolor{Red}{+22.37\%}$ \\
 & $8000$ & $20.37$ & $\geq 100$ & $22.49$ & $15.98$ & $\textcolor{Red}{+27.47\%}$ \\
\cline{3-7}
 & $\geq 16000$ & \multicolumn{5}{c|}{$\geq 20$ for all accountants} \\
\hline
\multirow{9}{*}{$2$} & $250$ & $1.45$ & $2.00$ & $1.55$ & $0.59$ & $\textcolor{Red}{+145.95\%}$ \\
 & $500$ & $1.53$ & $2.10$ & $1.61$ & $0.68$ & $\textcolor{Red}{+126.50\%}$ \\
 & $1000$ & $1.58$ & $2.28$ & $1.71$ & $0.80$ & $\textcolor{Red}{+96.51\%}$ \\
 & $2000$ & $1.65$ & $2.44$ & $1.91$ & $1.03$ & $\textcolor{Red}{+60.06\%}$ \\
 & $4000$ & $1.80$ & $2.68$ & $2.32$ & $1.43$ & $\textcolor{Red}{+25.92\%}$ \\
 & $8000$ & $2.10$ & $3.13$ & $3.12$ & $2.04$ & $\textcolor{Red}{+2.75\%}$ \\
 & $16000$ & $2.69$ & $3.67$ & $4.46$ & $2.94$ & $\textcolor{Green}{-8.75\%}$ \\
 & $32000$ & $3.82$ & $4.75$ & $6.49$ & $4.28$ & $\textcolor{Green}{-10.84\%}$ \\
 & $64000$ & $5.56$ & $6.53$ & $9.52$ & $6.27$ & $\textcolor{Green}{-11.35\%}$ \\
\hline
\multirow{9}{*}{$3$} & $250$ & $0.58$ & $0.67$ & $0.62$ & $0.18$ & $\textcolor{Red}{+225.05\%}$ \\
 & $500$ & $0.59$ & $0.69$ & $0.66$ & $0.25$ & $\textcolor{Red}{+136.65\%}$ \\
 & $1000$ & $0.62$ & $0.74$ & $0.76$ & $0.36$ & $\textcolor{Red}{+73.70\%}$ \\
 & $2000$ & $0.67$ & $0.79$ & $0.96$ & $0.51$ & $\textcolor{Red}{+31.47\%}$ \\
 & $4000$ & $0.77$ & $0.90$ & $1.33$ & $0.73$ & $\textcolor{Red}{+5.50\%}$ \\
 & $8000$ & $0.97$ & $1.11$ & $1.90$ & $1.05$ & $\textcolor{Green}{-7.39\%}$ \\
 & $16000$ & $1.36$ & $1.49$ & $2.74$ & $1.51$ & $\textcolor{Green}{-9.75\%}$ \\
 & $32000$ & $1.97$ & $2.11$ & $3.96$ & $2.19$ & $\textcolor{Green}{-9.97\%}$ \\
 & $64000$ & $2.85$ & $3.01$ & $5.77$ & $3.17$ & $\textcolor{Green}{-10.20\%}$ \\
\hline
\multirow{9}{*}{$4$} & $250$ & $0.31$ & $0.34$ & $0.34$ & $0.12$ & $\textcolor{Red}{+162.02\%}$ \\
 & $500$ & $0.32$ & $0.35$ & $0.38$ & $0.17$ & $\textcolor{Red}{+90.81\%}$ \\
 & $1000$ & $0.34$ & $0.38$ & $0.47$ & $0.24$ & $\textcolor{Red}{+42.81\%}$ \\
 & $2000$ & $0.38$ & $0.43$ & $0.65$ & $0.34$ & $\textcolor{Red}{+12.65\%}$ \\
 & $4000$ & $0.47$ & $0.52$ & $0.92$ & $0.49$ & $\textcolor{Green}{-3.41\%}$ \\
 & $8000$ & $0.65$ & $0.71$ & $1.32$ & $0.70$ & $\textcolor{Green}{-7.70\%}$ \\
 & $16000$ & $0.93$ & $0.99$ & $1.90$ & $1.01$ & $\textcolor{Green}{-7.84\%}$ \\
 & $32000$ & $1.34$ & $1.41$ & $2.74$ & $1.46$ & $\textcolor{Green}{-8.00\%}$ \\
 & $64000$ & $1.94$ & $2.01$ & $3.97$ & $2.11$ & $\textcolor{Green}{-8.22\%}$ \\
\hline
\multirow{9}{*}{$6$} & $250$ & $0.13$ & $0.14$ & $0.16$ & $0.07$ & $\textcolor{Red}{+87.97\%}$ \\
 & $500$ & $0.14$ & $0.15$ & $0.20$ & $0.10$ & $\textcolor{Red}{+41.64\%}$ \\
 & $1000$ & $0.16$ & $0.17$ & $0.28$ & $0.14$ & $\textcolor{Red}{+13.14\%}$ \\
 & $2000$ & $0.20$ & $0.22$ & $0.40$ & $0.20$ & $\textcolor{Green}{-1.18\%}$ \\
 & $4000$ & $0.28$ & $0.31$ & $0.58$ & $0.29$ & $\textcolor{Green}{-4.39\%}$ \\
 & $8000$ & $0.40$ & $0.43$ & $0.83$ & $0.42$ & $\textcolor{Green}{-4.59\%}$ \\
 & $16000$ & $0.58$ & $0.61$ & $1.19$ & $0.61$ & $\textcolor{Green}{-4.72\%}$ \\
 & $32000$ & $0.83$ & $0.86$ & $1.71$ & $0.87$ & $\textcolor{Green}{-4.85\%}$ \\
 & $64000$ & $1.20$ & $1.23$ & $2.47$ & $1.26$ & $\textcolor{Green}{-5.01\%}$ \\
\hline
\multirow{9}{*}{$8$} & $250$ & $0.07$ & $0.08$ & $0.11$ & $0.05$ & $\textcolor{Red}{+50.48\%}$ \\
 & $500$ & $0.08$ & $0.09$ & $0.15$ & $0.07$ & $\textcolor{Red}{+19.14\%}$ \\
 & $1000$ & $0.10$ & $0.12$ & $0.21$ & $0.10$ & $\textcolor{Red}{+2.65\%}$ \\
 & $2000$ & $0.14$ & $0.16$ & $0.30$ & $0.15$ & $\textcolor{Green}{-1.60\%}$ \\
 & $4000$ & $0.20$ & $0.22$ & $0.42$ & $0.21$ & $\textcolor{Green}{-2.12\%}$ \\
 & $8000$ & $0.29$ & $0.31$ & $0.60$ & $0.30$ & $\textcolor{Green}{-2.52\%}$ \\
 & $16000$ & $0.42$ & $0.44$ & $0.87$ & $0.43$ & $\textcolor{Green}{-2.64\%}$ \\
 & $32000$ & $0.60$ & $0.62$ & $1.24$ & $0.62$ & $\textcolor{Green}{-2.77\%}$ \\
 & $64000$ & $0.87$ & $0.89$ & $1.79$ & $0.89$ & $\textcolor{Green}{-2.90\%}$ \\
\hline
\end{tabular}
\end{table}

\begin{table}[H]
\centering
\caption{Accounted privacy $\varepsilon$ for $\delta=10^{-8}$ and $p=0.01$.}
\footnotesize
\begin{tabular}{|c|c|c|c|c|c|c|}
\hline
\multirow{2}{*}{$\sigma$} & \multirow{2}{*}{$T$} & \multicolumn{3}{c|}{R\'enyi accountants} & \multirow{2}{*}{Lossless accountants} & \multirow{2}{*}{Relative difference} \\
\cline{3-5}
 &  & Our bounds & Birrell et al. & Wang et al. &  &  \\
\hline
\multirow{3}{*}{$1$} & $250$ & $18.37$ & $\geq 100$ & $19.75$ & $13.45$ & $\textcolor{Red}{+36.60\%}$ \\
 & $500$ & $19.71$ & $\geq 100$ & $22.46$ & $16.45$ & $\textcolor{Red}{+19.82\%}$ \\
\cline{3-7}
 & $\geq 1000$ & \multicolumn{5}{c|}{$\geq 20$ for all accountants} \\
\hline
\multirow{9}{*}{$2$} & $250$ & $2.36$ & $4.19$ & $2.87$ & $1.89$ & $\textcolor{Red}{+24.53\%}$ \\
 & $500$ & $2.65$ & $5.11$ & $3.54$ & $2.33$ & $\textcolor{Red}{+13.65\%}$ \\
 & $1000$ & $3.09$ & $6.24$ & $4.79$ & $3.10$ & $\textcolor{Green}{-0.15\%}$ \\
 & $2000$ & $3.99$ & $7.53$ & $6.77$ & $4.35$ & $\textcolor{Green}{-8.28\%}$ \\
 & $4000$ & $5.59$ & $9.68$ & $9.85$ & $6.27$ & $\textcolor{Green}{-10.80\%}$ \\
 & $8000$ & $8.15$ & $13.19$ & $14.52$ & $9.20$ & $\textcolor{Green}{-11.39\%}$ \\
 & $16000$ & $12.01$ & $18.13$ & $21.61$ & $13.71$ & $\textcolor{Green}{-12.39\%}$ \\
 & $32000$ & $18.13$ & $27.83$ & $34.43$ & $20.79$ & $\textcolor{Green}{-12.77\%}$ \\
\cline{3-7}
 & $64000$ & \multicolumn{5}{c|}{$\geq 20$ for all accountants} \\
\hline
\multirow{9}{*}{$3$} & $250$ & $0.93$ & $1.27$ & $1.43$ & $0.76$ & $\textcolor{Red}{+22.59\%}$ \\
 & $500$ & $1.10$ & $1.52$ & $2.01$ & $1.07$ & $\textcolor{Red}{+3.14\%}$ \\
 & $1000$ & $1.41$ & $1.93$ & $2.85$ & $1.52$ & $\textcolor{Green}{-7.30\%}$ \\
 & $2000$ & $1.97$ & $2.57$ & $4.07$ & $2.19$ & $\textcolor{Green}{-9.80\%}$ \\
 & $4000$ & $2.85$ & $3.50$ & $5.89$ & $3.17$ & $\textcolor{Green}{-10.12\%}$ \\
 & $8000$ & $4.13$ & $4.88$ & $8.59$ & $4.61$ & $\textcolor{Green}{-10.48\%}$ \\
 & $16000$ & $6.04$ & $6.92$ & $12.70$ & $6.77$ & $\textcolor{Green}{-10.84\%}$ \\
 & $32000$ & $8.90$ & $10.08$ & $19.11$ & $10.04$ & $\textcolor{Green}{-11.40\%}$ \\
 & $64000$ & $13.26$ & $14.84$ & $29.96$ & $15.08$ & $\textcolor{Green}{-12.10\%}$ \\
\hline
\multirow{9}{*}{$4$} & $250$ & $0.54$ & $0.70$ & $0.97$ & $0.50$ & $\textcolor{Red}{+8.97\%}$ \\
 & $500$ & $0.69$ & $0.90$ & $1.37$ & $0.71$ & $\textcolor{Green}{-3.13\%}$ \\
 & $1000$ & $0.95$ & $1.19$ & $1.95$ & $1.01$ & $\textcolor{Green}{-6.21\%}$ \\
 & $2000$ & $1.35$ & $1.61$ & $2.79$ & $1.46$ & $\textcolor{Green}{-7.53\%}$ \\
 & $4000$ & $1.94$ & $2.23$ & $4.03$ & $2.11$ & $\textcolor{Green}{-8.14\%}$ \\
 & $8000$ & $2.80$ & $3.13$ & $5.85$ & $3.06$ & $\textcolor{Green}{-8.29\%}$ \\
 & $16000$ & $4.07$ & $4.45$ & $8.59$ & $4.46$ & $\textcolor{Green}{-8.65\%}$ \\
 & $32000$ & $5.95$ & $6.41$ & $12.74$ & $6.55$ & $\textcolor{Green}{-9.04\%}$ \\
 & $64000$ & $8.78$ & $9.34$ & $19.22$ & $9.70$ & $\textcolor{Green}{-9.47\%}$ \\
\hline
\multirow{9}{*}{$6$} & $250$ & $0.30$ & $0.39$ & $0.60$ & $0.30$ & $\textcolor{Green}{-0.19\%}$ \\
 & $500$ & $0.41$ & $0.52$ & $0.85$ & $0.42$ & $\textcolor{Green}{-2.01\%}$ \\
 & $1000$ & $0.59$ & $0.69$ & $1.21$ & $0.61$ & $\textcolor{Green}{-3.13\%}$ \\
 & $2000$ & $0.84$ & $0.95$ & $1.74$ & $0.87$ & $\textcolor{Green}{-4.01\%}$ \\
 & $4000$ & $1.20$ & $1.32$ & $2.49$ & $1.26$ & $\textcolor{Green}{-4.71\%}$ \\
 & $8000$ & $1.72$ & $1.86$ & $3.60$ & $1.81$ & $\textcolor{Green}{-5.06\%}$ \\
 & $16000$ & $2.49$ & $2.65$ & $5.25$ & $2.63$ & $\textcolor{Green}{-5.25\%}$ \\
 & $32000$ & $3.61$ & $3.80$ & $7.70$ & $3.82$ & $\textcolor{Green}{-5.50\%}$ \\
 & $64000$ & $5.29$ & $5.49$ & $11.44$ & $5.60$ & $\textcolor{Green}{-5.59\%}$ \\
\hline
\multirow{9}{*}{$8$} & $250$ & $0.21$ & $0.28$ & $0.43$ & $0.21$ & $\textcolor{Red}{+0.83\%}$ \\
 & $500$ & $0.30$ & $0.37$ & $0.62$ & $0.30$ & $\textcolor{Green}{-0.27\%}$ \\
 & $1000$ & $0.43$ & $0.49$ & $0.88$ & $0.43$ & $\textcolor{Green}{-1.14\%}$ \\
 & $2000$ & $0.61$ & $0.68$ & $1.26$ & $0.62$ & $\textcolor{Green}{-1.83\%}$ \\
 & $4000$ & $0.87$ & $0.95$ & $1.81$ & $0.89$ & $\textcolor{Green}{-2.37\%}$ \\
 & $8000$ & $1.25$ & $1.33$ & $2.61$ & $1.28$ & $\textcolor{Green}{-2.83\%}$ \\
 & $16000$ & $1.80$ & $1.89$ & $3.78$ & $1.85$ & $\textcolor{Green}{-3.02\%}$ \\
 & $32000$ & $2.60$ & $2.71$ & $5.51$ & $2.69$ & $\textcolor{Green}{-3.19\%}$ \\
 & $64000$ & $3.78$ & $3.90$ & $8.13$ & $3.91$ & $\textcolor{Green}{-3.43\%}$ \\
\hline
\end{tabular}
\end{table}

\begin{table}[H]
\centering
\caption{Accounted privacy $\varepsilon$ for $\delta=10^{-8}$ and $p=0.04$.}
\label{table:last}
\footnotesize
\begin{tabular}{|c|c|c|c|c|c|c|}
\hline
\multirow{2}{*}{$\sigma$} & \multirow{2}{*}{$T$} & \multicolumn{3}{c|}{R\'enyi accountants} & \multirow{2}{*}{Lossless accountants} & \multirow{2}{*}{Relative difference} \\
\cline{3-5}
 &  & Our bounds & Birrell et al. & Wang et al. &  &  \\
\hline
\multirow{1}{*}{$1$} & $\geq 250$ & \multicolumn{5}{c|}{$\geq 20$ for all accountants} \\
\hline
\multirow{5}{*}{$2$} & $250$ & $6.60$ & $20.32$ & $10.60$ & $6.67$ & $\textcolor{Green}{-1.07\%}$ \\
 & $500$ & $8.84$ & $25.49$ & $15.32$ & $9.31$ & $\textcolor{Green}{-5.05\%}$ \\
 & $1000$ & $12.29$ & $33.95$ & $22.39$ & $13.46$ & $\textcolor{Green}{-8.69\%}$ \\
 & $2000$ & $17.52$ & $50.86$ & $34.36$ & $20.03$ & $\textcolor{Green}{-12.53\%}$ \\
\cline{3-7}
 & $\geq 4000$ & \multicolumn{5}{c|}{$\geq 20$ for all accountants} \\
\hline
\multirow{7}{*}{$3$} & $250$ & $3.13$ & $5.78$ & $6.23$ & $3.24$ & $\textcolor{Green}{-3.18\%}$ \\
 & $500$ & $4.34$ & $7.57$ & $8.96$ & $4.62$ & $\textcolor{Green}{-6.19\%}$ \\
 & $1000$ & $6.14$ & $9.91$ & $13.09$ & $6.71$ & $\textcolor{Green}{-8.55\%}$ \\
 & $2000$ & $8.83$ & $13.79$ & $19.39$ & $9.90$ & $\textcolor{Green}{-10.79\%}$ \\
 & $4000$ & $13.06$ & $19.33$ & $30.53$ & $14.80$ & $\textcolor{Green}{-11.75\%}$ \\
 & $8000$ & $19.73$ & $29.57$ & $45.64$ & $22.53$ & $\textcolor{Green}{-12.44\%}$ \\
\cline{3-7}
 & $\geq 16000$ & \multicolumn{5}{c|}{$\geq 20$ for all accountants} \\
\hline
\multirow{8}{*}{$4$} & $250$ & $2.07$ & $3.22$ & $4.21$ & $2.13$ & $\textcolor{Green}{-2.59\%}$ \\
 & $500$ & $2.91$ & $4.20$ & $6.06$ & $3.06$ & $\textcolor{Green}{-4.71\%}$ \\
 & $1000$ & $4.15$ & $5.63$ & $8.78$ & $4.43$ & $\textcolor{Green}{-6.46\%}$ \\
 & $2000$ & $5.98$ & $7.77$ & $12.93$ & $6.49$ & $\textcolor{Green}{-7.88\%}$ \\
 & $4000$ & $8.72$ & $10.99$ & $19.36$ & $9.60$ & $\textcolor{Green}{-9.13\%}$ \\
 & $8000$ & $12.97$ & $16.08$ & $30.47$ & $14.38$ & $\textcolor{Green}{-9.82\%}$ \\
 & $16000$ & $19.61$ & $23.90$ & $46.09$ & $21.89$ & $\textcolor{Green}{-10.40\%}$ \\
\cline{3-7}
 & $\geq 32000$ & \multicolumn{5}{c|}{$\geq 20$ for all accountants} \\
\hline
\multirow{9}{*}{$6$} & $250$ & $1.25$ & $1.76$ & $2.57$ & $1.26$ & $\textcolor{Green}{-0.73\%}$ \\
 & $500$ & $1.77$ & $2.31$ & $3.69$ & $1.81$ & $\textcolor{Green}{-2.17\%}$ \\
 & $1000$ & $2.53$ & $3.12$ & $5.34$ & $2.62$ & $\textcolor{Green}{-3.36\%}$ \\
 & $2000$ & $3.64$ & $4.32$ & $7.82$ & $3.81$ & $\textcolor{Green}{-4.37\%}$ \\
 & $4000$ & $5.29$ & $6.11$ & $11.52$ & $5.57$ & $\textcolor{Green}{-5.04\%}$ \\
 & $8000$ & $7.75$ & $8.76$ & $17.36$ & $8.21$ & $\textcolor{Green}{-5.71\%}$ \\
 & $16000$ & $11.52$ & $12.79$ & $26.47$ & $12.26$ & $\textcolor{Green}{-5.96\%}$ \\
 & $32000$ & $17.46$ & $19.02$ & $41.09$ & $18.54$ & $\textcolor{Green}{-5.87\%}$ \\
\cline{3-7}
 & $64000$ & \multicolumn{5}{c|}{$\geq 20$ for all accountants} \\
\hline
\multirow{9}{*}{$8$} & $250$ & $0.90$ & $1.22$ & $1.86$ & $0.89$ & $\textcolor{Red}{+0.64\%}$ \\
 & $500$ & $1.28$ & $1.62$ & $2.66$ & $1.28$ & $\textcolor{Green}{-0.48\%}$ \\
 & $1000$ & $1.82$ & $2.19$ & $3.84$ & $1.85$ & $\textcolor{Green}{-1.36\%}$ \\
 & $2000$ & $2.62$ & $3.02$ & $5.57$ & $2.68$ & $\textcolor{Green}{-2.14\%}$ \\
 & $4000$ & $3.79$ & $4.26$ & $8.18$ & $3.90$ & $\textcolor{Green}{-2.72\%}$ \\
 & $8000$ & $5.52$ & $6.10$ & $12.08$ & $5.71$ & $\textcolor{Green}{-3.33\%}$ \\
 & $16000$ & $8.15$ & $8.81$ & $18.22$ & $8.43$ & $\textcolor{Green}{-3.29\%}$ \\
 & $32000$ & $12.07$ & $12.96$ & $28.19$ & $12.59$ & $\textcolor{Green}{-4.13\%}$ \\
 & $64000$ & $18.27$ & $19.38$ & $43.45$ & $19.07$ & $\textcolor{Green}{-4.18\%}$ \\
\hline
\end{tabular}
\end{table}

\subsection{Numerical implementation}\label{subsec:numerics-implementation}

Here, we record several details related to the numerical implementation of privacy accountants used for Appendix~\ref{subsec:numerics-full}.

\paragraph*{Implementing R\'enyi accountants.} For R\'enyi accountants, we evaluate the R\'enyi DP bounds for integer orders $2 \leq \alpha \leq 256$, convert each bound into $(\eps, \delta)$-DP bound using \citep[Theorem 20]{balle2020hypothesis}, and choose the minimum $\eps$ among them. For \citep{birrell2024rdp}, we calculate R\'enyi DP bounds with tuning orders $m \in \{3, 4, 5\}$ (the values used throughout numerical results in \citep{birrell2024rdp})  and report the minimum among the converted $\eps$.

Our R\'enyi accountant requires\footnote{The numerical optimization over $u \in [0, 1/2]$ can be bypassed under conditions in Proposition~\ref{prop:sufficient}. In our numerical implementation, for simplicity we only check \eqref{eq:optunique-largealpha}.} calculating the maximum of R\'enyi divergences over $u \in [0, 1/2]$, as stated in \eqref{eq:1dopt}. The evaluation at each $u$ is done by first evaluating $F(u)$ \eqref{eq:ffunction}, and then converting that to the corresponding R\'enyi divergence $J(u) \triangleq (\log F(u)) / (\alpha - 1)$. In particular, we express \eqref{eq:ffunction} as a one-dimensional integral and characterize a valid upper bound $\overline{J}(u)$ and lower bound $\underline{J}(u)$ for $J(u)$ (see Appendix~\ref{subsubsec:pointcalculation} for details). Based on these, we use a branch-and-bound approach as follows.

\begin{itemize}
    \item Initialization: $[\ell, r] \gets [0, 1/2]$ stored in a max heap, whose key $Q(\ell, r)$  satisfies $Q(\ell, r) \geq \max_{u \in [\ell, r]}J(u)$ (calculating the key requires $\overline{J}(\ell), \overline{J}(r)$ and a ``curvature'' upper bound; see Appendix~\ref{subsubsec:maximumcalculation} for details); a global lower bound $B \gets \underline{J}(0) \vee \underline{J}(1/2)$.
    \item Update: at each iteration,
    \begin{enumerate}
        \item If the heap is empty or the key of the maximum item minus $B$ is at most a predetermined threshold $\gamma$, terminate and output the key ($B$ if the heap is empty).
        \item Otherwise, pop the maximum item $[\ell, r]$ from the heap and calculate $\overline{J}(m)$ and $\underline{J}(m)$ for $m \gets (\ell + r)/2$. Update the global lower bound $B \gets B \vee \underline{J}(m)$.
        \item For each interval-key pair $([\ell, m], Q(\ell,m))$ and $([m, r], Q(m, r))$, push it to the heap if and only if its key is larger than $B$.
    \end{enumerate}
\end{itemize}
In particular, it is guaranteed that the output is at least as large as the theoretical R\'enyi DP bound, and has an additive error of at most $\gamma$; see Appendix~\ref{subsubsec:maximumcalculation} for details.

Before this procedure for evaluation of \eqref{eq:1dopt} at a particular $\alpha$, we first (approximately) calculate \eqref{eq:1dopt} using standard libraries such as \texttt{SciPy}, in order to choose the R\'enyi order $\alpha$. As described in Proposition~\ref{prop:unimodal}, this can be done with $O(\log \alpha_{\max})$ evaluations, and here we use the approximations of those evaluations for faster implementation. For the chosen $\alpha$, we implement the aforementioned branch-and-bound procedure with ball arithmetic from \texttt{python-flint}, guaranteeing that the final numerical value (based on the upper endpoint from ball arithmetic) is at least as large as the theoretical value.

\paragraph*{Implementing lossless accountants.} For lossless accountants, we use the algorithm from \citep{doroshenko2022connect} for implementation, with the length of discretization interval set as $10^{-4}$.

\paragraph*{Runtime analysis.} In order to attain an additive error of at most $\gamma$ (in R\'enyi DP, for each R\'enyi order $\alpha$) for the branch-and-bound approach, it suffices to have $O(\alpha_{\max}^{3/2}/(\sigma^2 \sqrt{\gamma}))$ many evaluations of the functions $\overline{J}$, $\underline{J}$ and the key $Q$; see Appendix~\ref{subsec:numerics-derivation} for further details. 

While our accountant is based on multiple steps including branch-and-bound and numerical integrals, its implementation in practice can be done reasonably fast. This is in part due to the fact that only logarithmically many R\'enyi orders need to be evaluated (Proposition~\ref{prop:unimodal}).  In Figure~\ref{fig:runtime}, we record the median runtimes (with 10th and 90th percentiles as error bars) from 20 runs of the R\'enyi accountants with $\alpha_{\max} = 64$, where our accountant runs within time comparable to that of \citep{wang2021rdp}; the runtimes were measured on a consumer-grade laptop with Intel Core i7-13700H CPU and 16GB LPDDR5 RAM.  We emphasize that our implementation of other R\'enyi accountants is based on specific codebases (\citep{wang2021rdp} from Google's \texttt{dp-accounting} and \citep{birrell2024rdp} from their codebase; note that all integer R\'enyi orders $2 \leq \alpha \leq 64$ are evaluated for those) which we did not further optimize. Furthermore, in general the runtimes may vary across different choices of $T, p, \sigma, \delta$ and $\alpha_{\max}$. Our only claim here is that our accountant can be run within a reasonable amount of time, and the measured runtimes for other R\'enyi accountants in the figure should only be considered as a reference of scale.

\begin{figure*}[t]
    \centering
    \includegraphics[width=0.45\textwidth]{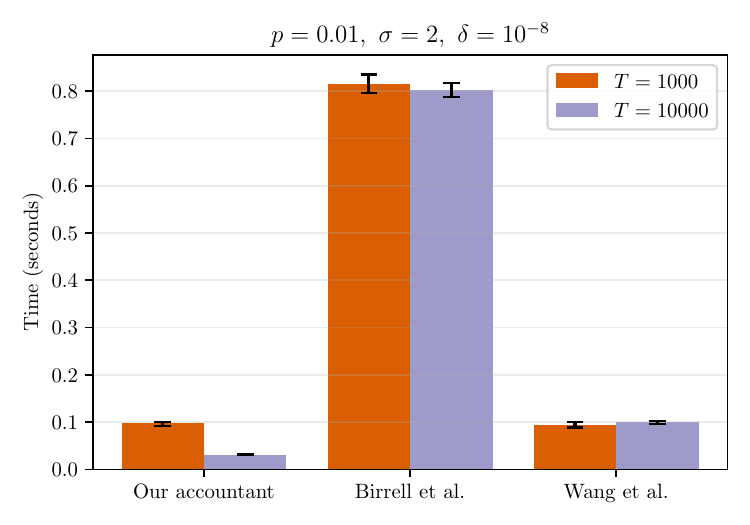}
    \hfill
    \includegraphics[width=0.45\textwidth]{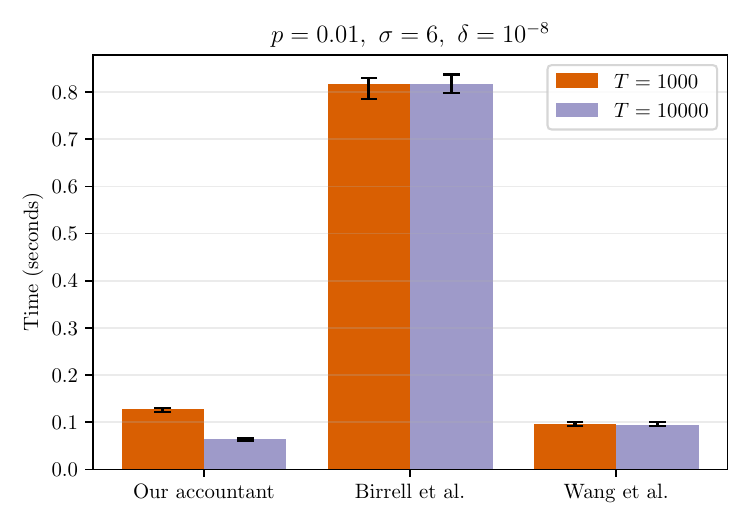}
    \caption{Median runtimes of R\'enyi accountants, for $p = 0.01, \sigma \in \{2, 6\}, \delta = 10^{-8}$, and $T \in \{1000, 10000\}$. The error bars represent the 10th and 90th percentiles. In this setting, our accountant runs faster for $T = 10000$ because the binary search based on the unimodality in R\'enyi order (see Proposition~\ref{prop:unimodal}) terminates earlier.}
    \label{fig:runtime}
\end{figure*}

\subsection{Mathematical derivations}\label{subsec:numerics-derivation}

Here, we record remaining mathematical details about calculating values related to $J(u) = (\log F(u))/(\alpha-1)$ (see Appendix~\ref{subsec:numerics-implementation}): recall that $F$ is defined as in \eqref{eq:ffunction} and the R\'enyi DP bound \eqref{eq:1dopt} is equal to $\max_{u \in [0, 1/2]} J(u)$.
Throughout, we assume that $\alpha \geq 2$ is an integer. 
\subsubsection{Point calculation}\label{subsubsec:pointcalculation}

Here, we explain how the upper bound $\overline{J}(u)$ and the lower bound $\underline{J}(u)$ of $J(u)$ are evaluated (see Appendix~\ref{subsec:numerics-implementation} for how they are used). Recall the notations around $F(u)$ from \eqref{eq:ffunction}. Define $\beta \triangleq \Delta \sqrt{2u}$ and $V \triangleq \langle b, Z \rangle / \beta \sim \calN(0, 1)$. Then from $\Var[\langle a, Z \rangle] = \Delta^2$ and $\Cov[\langle a, Z\rangle, \langle b, Z\rangle] = \Delta^2u = \beta^2/2$, we have
\[\langle a, Z \rangle \overset{\mathrm{d}}{=} \frac{\beta}{2} V + \sqrt{\Delta^2 - \frac{\beta^2}{4}}W\,,
\]
where $W \sim \calN(0, 1)$ is independent of $V$; note that such reparametrization is valid even if $u = 0$. Then
\begin{equation}\label{eq:renyinum-expansion}
\begin{aligned}
F(u) &= \E[(1-p + pe^{(\beta/2)V + \sqrt{\Delta^2 - \beta^2/4}W - \Delta^2/2})^{\alpha} (1-p + pe^{\beta V - \beta^2/2})^{1-\alpha}] \\
    &= \sum_{k=0}^\alpha \binom{\alpha}{k} (1-p)^{\alpha-k}p^k \E[(1-p + pe^{\beta V - \beta^2/2})^{1-\alpha}e^{(k\beta/2)V + k\sqrt{\Delta^2 - \beta^2/4}W - k\Delta^2/2}] \\
    &= \sum_{k=0}^\alpha \binom{\alpha}{k} (1-p)^{\alpha-k}p^k e^{k(k-1)\Delta^2/2 - k^2\beta^2/8}\E[(1-p + pe^{\beta V - \beta^2/2})^{1-\alpha}e^{(k\beta/2)V}] \\
    &= \sum_{k=0}^\alpha \binom{\alpha}{k} (1-p)^{\alpha-k}p^k e^{k(k-1)\Delta^2/2 - k^2\beta^2/8}\int_{-\infty}^\infty (1-p + pe^{\beta v - \beta^2/2})^{1-\alpha}e^{(k\beta/2)v} \phi(v) \diff v\,,
\end{aligned}
\end{equation}
where we are left with one-dimensional integrals. For each $k$, we decompose the integral into three parts with truncation:
\begin{align*}
    \int_{-\infty}^\infty (1-p + pe^{\beta v - \beta^2/2})^{1-\alpha}e^{(k\beta/2)v} \phi(v) \diff v &=  \underbrace{\int_{-L}^R (1-p + pe^{\beta v - \beta^2/2})^{1-\alpha}e^{(k\beta/2)v} \phi(v) \diff v}_{\mathrm{(I)}} \\
    &\quad + \underbrace{\int_{-\infty}^{-L} (1-p + pe^{\beta v - \beta^2/2})^{1-\alpha}e^{(k\beta/2)v} \phi(v) \diff v}_{\mathrm{(II)}} \\
    &\quad+ \underbrace{\int_{R}^\infty (1-p + pe^{\beta v - \beta^2/2})^{1-\alpha}e^{(k\beta/2)v} \phi(v) \diff v}_{\mathrm{(III)}}\,,
\end{align*}
for some $L, R > 0$ chosen to meet the desired accuracy (in our implementation, we start with $L = R = 8$ and increase them if needed); see also the final paragraph in Appendix~\ref{subsubsec:maximumcalculation} for the accuracy criterion.

\paragraph*{Lower bound $\underline{J}(u)$.} We only use $\mathrm{(I)}$ for \eqref{eq:renyinum-expansion}, which yields a valid lower bound $\underline{F}(u)$ of $F(u)$:
\[\underline{F}(u) \triangleq \int_{-L}^R\left(\sum_{k=0}^\alpha \binom{\alpha}{k} (1-p)^{\alpha-k}p^k e^{k(k-1)\Delta^2/2 - k^2\beta^2/8} (1-p + pe^{\beta v - \beta^2/2})^{1-\alpha}e^{(k\beta/2)v} \phi(v) \right)\diff v\,. \]\

Then a lower bound $\underline{J}(u)$ of $J(u)$ is defined as $\underline{J}(u) \triangleq (\log  \underline{F}(u))/(\alpha-1)$.

\paragraph*{Upper bound $\overline{J}(u)$.} For $\mathrm{(II)}$ and $\mathrm{(III)}$, we use two simple upper bounds for the factor $(1-p + pe^{\beta V - \beta^2/2})^{1-\alpha}$; namely, $(1-p)^{1-\alpha}$ and $p^{1-\alpha}e^{(1-\alpha)\beta V - (1-\alpha)\beta^2/2}$. Those yield
\begin{align*}
    \mathrm{(II)} &\leq \int_{-\infty}^{-L} (1-p)^{1-\alpha}e^{(k\beta/2)v} \phi(v) \diff v \wedge \int_{-\infty}^{-L} p^{1-\alpha}e^{-(1-\alpha)\beta^2/2} e^{(k/2 + 1-\alpha)\beta v} \phi(v) \diff v \\
    &= (1-p)^{1-\alpha} e^{k^2\beta^2/8}\Phi\left(-L - \frac{k\beta}{2}\right) \wedge p^{1-\alpha}e^{k^2\beta^2/8 + (\alpha-1)(\alpha-k)\beta^2/2}\Phi\left(-L - \left(\frac{k}{2} + 1 -\alpha\right)\beta\right)\,,\\
    \mathrm{(III)} &\leq \int_{R}^{\infty} (1-p)^{1-\alpha}e^{(k\beta/2)v} \phi(v) \diff v \wedge \int_{R}^{\infty} p^{1-\alpha}e^{-(1-\alpha)\beta^2/2} e^{(k/2 + 1-\alpha)\beta v} \phi(v) \diff v \\
    &= (1-p)^{1-\alpha}e^{k^2\beta^2/8}\Phi\left(-R + \frac{k\beta}{2}\right) \wedge p^{1-\alpha}e^{k^2\beta^2/8 + (\alpha-1)(\alpha-k)\beta^2/2}\Phi\left(-R + \left(\frac{k}{2} + 1 -\alpha\right)\beta\right)\,.
\end{align*}
Hence for \eqref{eq:renyinum-expansion}, we have a valid upper bound $\overline{F}(u)$ of $F(u)$:
\begin{align*}
    \overline{F}(u) &\triangleq \underline{F}(u) + \sum_{k=0}^\alpha \Bigg(\binom{\alpha}{k}(1-p)^{1-k}p^k e^{k(k-1)\Delta^2/2} \Phi\left(-L - \frac{k\beta}{2}\right) \\
    &\quad \quad \wedge \binom{\alpha}{k}(1-p)^{\alpha-k}p^{k+1-\alpha}e^{k(k-1)\Delta^2/2 + (\alpha-1)(\alpha-k) \beta^2/2}\Phi\left(-L - \left(\frac{k}{2} + 1 -\alpha\right)\beta\right) \Bigg) \\
    &\quad+ \sum_{k=0}^\alpha \Bigg(\binom{\alpha}{k}(1-p)^{1-k}p^k e^{k(k-1)\Delta^2/2} \Phi\left(-R + \frac{k\beta}{2}\right) \\
    &\quad \quad \wedge \binom{\alpha}{k}(1-p)^{\alpha-k}p^{k+1-\alpha}e^{k(k-1)\Delta^2/2 + (\alpha-1)(\alpha-k) \beta^2/2}\Phi\left(-R + \left(\frac{k}{2} + 1 -\alpha\right)\beta\right) \Bigg)\,.
\end{align*}
Then an upper bound $\overline{J}(u)$ of $J(u)$ is defined as $\overline{J}(u) \triangleq (\log \overline{F}(u))/(\alpha - 1)$.

\subsubsection{Maximum calculation}\label{subsubsec:maximumcalculation}

Here, we explain how upper bound $Q(\ell, r)$ of $\max_{u \in [\ell, r]} J(u)$ (see Appendix~\ref{subsec:numerics-implementation}) is evaluated. First, for all $t \in [0, 1]$ we have
\begin{equation}\label{eq:linearinterpolation}
\begin{aligned}
    J(\ell + t(r-\ell)) &\leq (1-t)J(\ell) + tJ(r) + \max_{u \in [\ell, r]}|J''(u)|\frac{(r-\ell)^2}{2}t(1-t)\\
    &\leq (1-t)\overline{J}(\ell) + t\overline{J}(r) + \max_{u \in [\ell, r]}|J''(u)|\frac{(r-\ell)^2}{2}t(1-t)\,,
\end{aligned}
\end{equation}
where the first inequality is a standard result from linear interpolation, and the second inequality is from $\overline{J} \geq J$ (see Appendix~\ref{subsubsec:pointcalculation}). Thus, once we can evaluate a valid upper bound $\kappa(\ell, r)$ of $\max_{u \in [\ell, r]}|J''(u)|$, we can set
\begin{equation}\label{eq:keydef}
    Q(\ell, r) \triangleq \max_{t \in [0, 1]} \left\{(1-t)\overline{J}(\ell) + t\overline{J}(r) + \kappa(\ell, r)\frac{(r-\ell)^2}{2}t(1-t)\right\}\,,
\end{equation}
where the maximum over $t \in [0, 1]$ can be evaluated as a quadratic function.

For calculating $\kappa(\ell, r)$, first note that
\begin{equation}\label{eq:secondderiv-setup}
    |J''(u)| = \frac{1}{\alpha-1}\left|\frac{F''(u)}{F(u)} - \left(\frac{F'(u)}{F(u)}\right)^2\right| \leq \frac{1}{\alpha-1}\left(\left|\frac{F''(u)}{F(u)}\right| + \left|\frac{F'(u)}{F(u)}\right|^2 \right)\,.
\end{equation}
For bounding $|F'(u)|$ and $|F''(u)|$, from \eqref{eq:firstderiv} and \eqref{eq:secondderiv} we have (see Appendix~\ref{subsec:derivs} for the notations below)
\begin{equation}\label{eq:firstderiv-abs}
    |F'(u)| \leq \Delta^2\alpha(\alpha-1)(\E[wvh(X, Y)] + \E[w^2h(X, Y)])\,,
\end{equation}
and
\begin{equation}\label{eq:secondderiv-abs}
\begin{aligned}
    |F''(u)| &\leq \Delta^4\alpha(\alpha-1)\Big((\alpha+1)(\alpha+2) \E[w^4 h(X, Y)] + 2\alpha(\alpha+1)\E[w^3 v h(X, Y)]\\
    &\quad + \alpha(\alpha-1)\E[w^2v^2h(X, Y)] + 4(\alpha+1)\E[w^3h(X, Y)] + 5\alpha\E[w^2vh(X, Y)] \\
    &\quad+ (\alpha - 1)\E[wv^2h(X, Y)] + 2\E[w^2 h(X, Y)] + \E[wvh(X, Y)]\Big)\,.
\end{aligned}    
\end{equation}
From these, we derive two upper bounds $\kappa_1(\ell, r)$ and $\kappa_2(\ell, r)$ for $\max_{u \in [\ell, r]}|J''(u)|$, from which we take $\kappa(\ell, r) \triangleq \kappa_1(\ell, r) \wedge \kappa_2(\ell, r)$.

\paragraph*{First bound $\kappa_1(\ell, r)$.} From \eqref{eq:secondderiv-setup}, since $F(u) \geq 1$ we have
\begin{equation}\label{eq:firstbound-setup}
    |J''(u)| \leq \frac{1}{\alpha - 1}(|F''(u)| + |F'(u)|^2)\,.
\end{equation}
For bounding \eqref{eq:firstderiv-abs} and \eqref{eq:secondderiv-abs}, it suffices to bound $\E[v^i w^jh(X, Y)]$ for finitely many $(i, j)$. This is done similarly as in Appendix~\ref{subsubsec:pointcalculation}, where the denominator in the power of $1 - p + pe^Y$ is bounded by $1-p$ and $pe^Y$. To be specific, we have\footnote{The binomial expansion of $(1-p + pe^X)^{\alpha-i}$ is valid since $i \leq 2$ for the right hand sides of \eqref{eq:firstderiv-abs} and \eqref{eq:secondderiv-abs}, and $\alpha \geq 2$.}
\begin{equation}\label{eq:secondderiv-monomial}
\begin{aligned}
    \E[v^iw^jh(X, Y)] &= \E[p^{i+j}e^{iX + jY}(1-p + pe^X)^{\alpha - i}(1-p + pe^Y)^{1-\alpha - j}] \\
    &= \sum_{k=0}^{\alpha-i}\binom{\alpha-i}{k}(1-p)^{\alpha-i-k}p^{k+i+j}\E[e^{(i+k)X + jY}(1-p+pe^Y)^{1-\alpha-j}] \\
    &\leq \sum_{k=0}^{\alpha-i}\binom{\alpha-i}{k}(1-p)^{\alpha-i-k}p^{k+i+j}\left(\E[e^{(i+k)X + jY}(1-p)^{1-\alpha-j}] \wedge \E[e^{(i+k)X + jY}(pe^Y)^{1-\alpha-j}]\right) \\
    &\leq \sum_{k=0}^{\alpha-i} \Bigg(\binom{\alpha-i}{k}(1-p)^{1-i-j-k}p^{i+j+k}e^{(\Delta^2/2)((i+k)(i+k-1) + 2rj(i+j+k-1))} \\
    &\quad \wedge \binom{\alpha-i}{k}(1-p)^{\alpha-i-k}p^{1 + k + i - \alpha} e^{(\Delta^2/2)((i+k)(i+k-1) + 2r(\alpha-1)(\alpha - i - k))}\Bigg)\,.
\end{aligned}    
\end{equation}

Then $\kappa_1(\ell, r)$ is defined as the final bound from the chain of inequalities (i.e., from sequentially applying the final upper bounds) in \eqref{eq:firstbound-setup}, \eqref{eq:firstderiv-abs}, \eqref{eq:secondderiv-abs}, and \eqref{eq:secondderiv-monomial}.

\paragraph*{Second bound $\kappa_2(\ell, r)$.} Consider a tilted measure 
\[\E'[f(X, Y)] \triangleq \frac{\E[f(X, Y)h(X, Y)]}{\E[h(X, Y)]} \text{ for }f \geq 0\,.
\]
Since $F(u) = \E[h(X, Y)]$, for \eqref{eq:firstderiv-abs} we have
\begin{equation}\label{eq:secondbound-first}
    \left|\frac{F'(u)}{F(u)}\right| \leq \Delta^2\alpha(\alpha-1)(\E'[wv] + \E'[w^2]) \leq 2\Delta^2\alpha(\alpha-1)\,,
\end{equation}
from $0 \leq v, w \leq 1$; similarly, for \eqref{eq:secondderiv-abs} we have
\begin{equation}\label{eq:secondbound-second}
\begin{aligned}
    \left|\frac{F''(u)}{F(u)}\right| &\leq \Delta^4\alpha(\alpha-1)\left((\alpha+1)(\alpha+2)  + 2\alpha(\alpha+1) + \alpha(\alpha-1) + 4(\alpha+1) + 5\alpha + (\alpha-1) + 2 + 1\right) \\
    &= 2\Delta^4\alpha(\alpha-1)(2\alpha^2 + 7\alpha + 4)\,.
\end{aligned}
\end{equation}
Thus, from \eqref{eq:secondderiv-setup} we have
\begin{equation}\label{eq:secondbound}
    |J''(u)| \leq \frac{1}{\alpha-1}(4\Delta^4\alpha^2(\alpha-1)^2 + 2\Delta^4\alpha(\alpha-1)(2\alpha^2 + 7\alpha + 4)) = 2\Delta^4\alpha(4\alpha^2 + 5\alpha + 4) \triangleq \kappa_2(\ell, r)\,.
\end{equation}

\paragraph*{Remaining details for implementation.} For our implementation of the branch-and-bound approach (see Appendix~\ref{subsec:numerics-implementation}), we choose $\overline{J}(u)$ and $\underline{J}(u)$ to satisfy $\overline{J}(u) - \underline{J}(u) \leq \gamma / 2$. Then the key $Q(\ell, r)$ from the maximum item $[\ell, r]$ satisfies (see \eqref{eq:keydef})
\[Q(\ell, r) \leq (\underline{J}(\ell) \vee 
\underline{J}(r)) + \frac{\gamma}{2} + \kappa(\ell, r)\frac{(r-\ell)^2}{8}\,,
\]
which in particular implies
\[Q(\ell, r) - B \leq \frac{\gamma}{2} + \kappa(\ell, r)\frac{(r-\ell)^2}{8}\,.
\]
Thus, if 
\begin{equation}\label{eq:errorcond}
    \kappa(\ell, r)\frac{(r - \ell)^2}{8} \leq \frac{\gamma}{2}\,,
\end{equation}
then $Q(\ell, r) - B \leq \gamma$. Since $B \leq \max_{u \in [0, 1/2]}J(u) \leq Q(\ell, r)$ (if the heap is empty then $B = \max_{u \in [0, 1/2]}J(u)$ is the output), this guarantees that the output from the branch-and-bound approach has an additive error of $\gamma$. From \eqref{eq:secondbound}, we know that $\kappa(\ell, r) = O(\Delta^4\alpha_{\max}^3) = O(\alpha_{\max}^3/\sigma^4)$, and thus for \eqref{eq:errorcond} it suffices to have $r - \ell = O(\sqrt{\gamma}\sigma^2/\alpha_{\max}^{3/2})$. This implies that for the bisection in the branch-and-bound approach, it needs at most $O(\alpha_{\max}^{3/2}/(\sigma^2 \sqrt{\gamma}))$ evaluations of $\overline{J}$, $\underline{J}$ and $Q$.

\end{document}